%% file: nash.tex
\documentclass[11pt]{article}
\usepackage[margin=0.8in]{geometry}

\usepackage{graphicx}
\usepackage{framed}
\usepackage[normalem]{ulem}

\usepackage{xfrac, amsthm,thm-restate,thmtools}
\usepackage{amssymb}
\usepackage{amsbsy}
\usepackage{amsfonts}
\usepackage{enumerate}
\usepackage{etoolbox}
\usepackage{subcaption}
\usepackage{url}
\usepackage{enumitem}
\usepackage{mathtools}
\usepackage{comment}
\usepackage[font={small,it}]{caption}
\usepackage{bbm}
\usepackage[usenames,dvipsnames]{color}
\usepackage{xcolor}
\usepackage{algorithmic} 
\usepackage{algorithm}
\usepackage{array}
\usepackage{nicefrac}
\usepackage{bm}
\usepackage{tikz}

\usepackage[colorlinks=true, allcolors=black]{hyperref}
\hypersetup{
  colorlinks   = true, 
  urlcolor     = blue, 
  linkcolor    = blue, 
  citecolor   = blue 
}

\usepackage{thmtools,thm-restate}

\DeclareMathOperator*{\argmax}{arg\,max}

\newtheorem{theorem}{Theorem}
\newtheorem{lemma}{Lemma}
\newtheorem{claim}[lemma]{Claim}
\newtheorem{corollary}[lemma]{Corollary}

\newcommand{\arm}{\mathtt{arm}}

\newcommand{\E}{\mathbb{E}}
\newcommand{\prob}{\mathbb{P}}

\newcommand{\NRg}{\textsc{NR}}
\newcommand{\Rg}{\textsc{R}}

\newcommand{\NCB}{\mathrm{NCB}}
\newcommand{\UCB}{\mathrm{UCB}}

\newcommand{\Sample}{S}

\newcommand*{\textcal}[1]{%
  \textit{\fontfamily{qzc}\selectfont#1}%
}

\title{\bfseries Fairness and Welfare Quantification for Regret in \\ Multi-Armed Bandits}

\author{Siddharth Barman\thanks{Indian Institute of Science. {\tt barman@iisc.ac.in}} \quad Arindam Khan\thanks{Indian Institute of Science. {\tt arindamkhan@iisc.ac.in}} \quad Arnab Maiti \thanks{Indian Institute of Technology, Kharagpur. {\tt maitiarnab9@gmail.com}} \quad Ayush Sawarni\thanks{Indian Institute of Science. {\tt ayushsawarni@iisc.ac.in}}}
\date{}

\begin{document}
\maketitle

\input{abstract}

\input{intro}

\input{notation}

\input{algorithm}

\input{improved+anytime-algorithm}

\input{conclusion}

\bibliographystyle{alpha}
\bibliography{nash-regret,fairMABpapers}

\newpage 
\appendix
\input{appendix-prob-G}

\input{appendix-numeric}
\input{appendix-key-lemmas}

\input{appendix-improved-anytime}

\input{appendix-nash-defns}
\input{ucb-bad}
\end{document}

%% file: abstract.tex
\begin{abstract}
We extend the notion of regret with a welfarist perspective. Focussing on the classic multi-armed bandit (MAB) framework, the current work quantifies the performance of bandit algorithms by applying a fundamental welfare function, namely the Nash social welfare (NSW) function. This corresponds to equating algorithm's performance to the geometric mean of its expected rewards and leads us to the study of  {\it Nash regret}, defined as the difference between the---a priori unknown---optimal mean (among the arms) and the algorithm's performance. Since NSW is known to satisfy fairness axioms, our approach complements the utilitarian considerations of average (cumulative) regret, wherein the algorithm is evaluated via the arithmetic mean of its expected rewards. 

This work develops an algorithm that, given the horizon of play $T$, achieves a Nash regret of $O \left( \sqrt{\frac{{k \log T}}{T}} \right)$, here $k$ denotes the number of arms in the MAB instance. Since, for any algorithm, the Nash regret is at least as much as its average regret (the AM-GM inequality), the known lower bound on average regret holds for Nash regret as well. Therefore, our Nash regret guarantee is essentially tight.  In addition, we develop an anytime algorithm with a Nash regret guarantee of $O \left( \sqrt{\frac{{k\log T}}{T}} \log T \right)$.   
\end{abstract}

%% file: intro.tex
\section{Introduction}
Regret minimization is a pre-eminent objective in the study of decision making under uncertainty. Indeed, regret is a central notion in multi-armed bandits \cite{lattimore2020bandit}, reinforcement learning \cite{sutton2018reinforcement}, game theory \cite{young2004strategic}, decision theory \cite{peterson2017introduction}, and causal inference \cite{lattimore2016causal}. The current work extends the formulation of regret with a welfarist perspective. In particular, we quantify the performance of a decision maker by applying a fundamental welfare function---namely the Nash social welfare---and study {Nash regret}, defined as the difference between the (a priori unknown) optimum and the decision maker's performance. We obtain essentially tight upper bounds for Nash regret in the classic multi-armed bandit (MAB) framework.  

Recall that the MAB framework provides an encapsulating abstraction for settings that entail sequential decision making under uncertainty. In this framework, a decision maker (online algorithm) has sample access to $k$ distributions (arms), which are a priori unknown. For $T$ rounds, the online algorithm sequentially selects an arm and receives a reward drawn independently from the arm-specific distribution. Here, ex ante, the optimal solution would be to select, in every round, the arm with the maximum mean. However, since the statistical properties of the arms are unknown, the algorithm accrues---in the $T$ rounds---expected rewards that are not necessarily the optimal. The construct of regret captures this sub-optimality and, hence, serves as a key performance metric for algorithms. A bit more formally, regret is defined as the difference between the optimal mean (among the arms) and the algorithm's performance.  

Two standard forms of regret are average (cumulative) regret \cite{lattimore2016causal} and instantaneous (simple) regret \cite{slivkins2019introduction}. Specifically, average regret considers the difference between the optimal mean and the average (arithmetic mean) of the expected rewards accumulated by the algorithm. Hence, in average regret the algorithm's performance is quantified as the arithmetic mean of the expected rewards it accumulates in the $T$ rounds. Complementarily, in simple regret the algorithm's performance is equated to its expected reward precisely in the $T$th round, i.e., the algorithm's performance is gauged only after $(T-1)$ rounds. 

We build upon these two formulations with a welfarist perspective. To appreciate the relevance of this perspective in the MAB context, consider settings in which the algorithm's rewards correspond to values that are distributed across a population of agents. In particular, one can revisit the classic motivating example of clinical trials \cite{thompson1933likelihood}: in each round $t \in \{1, \ldots, T\}$ the decision maker (online algorithm) administers one of the $k$ drugs to the $t$th patient. The observed reward in round $t$ is the selected drug's efficacy for patient $t$. Hence, in average regret one equates the algorithm's performance as the (average) social welfare across the $T$ patients. It is pertinent to note that maximizing social welfare (equivalently, minimizing average regret) might not induce a fair outcome. The social welfare can in fact be high even if the drugs are inconsiderately tested on an initial set of patients. By contrast, in instantaneous regret, the algorithm's performance maps to the egalitarian (Rawlsian) welfare (which is a well-studied fairness criterion) but only after excluding an initial set of test subjects.   

The work aims to incorporate fairness and welfare considerations for such MAB settings, in which the algorithm's rewards correspond to agents' values (e.g., drug efficacy, users' experience, and advertisers' revenue). Towards this, we invoke a principled approach from mathematical economics: apply an axiomatically-justified welfare function on the values and thereby quantify the algorithm's performance. Specifically, we apply the Nash social welfare (NSW), which is defined as the geometric mean of agents' values \cite{moulin2004fair}. Hence, we equate the algorithm's performance to the geometric mean of its $T$ expected rewards. This leads us to the notion of Nash regret, defined as the difference between the optimal mean, $\mu^*$, and the geometric mean of the expected rewards (see equation (\ref{eq:def-NRg})). Note that here we are conforming to an ex-ante assessment--the benchmark (in particular, $\mu^*$) is an expected value and the value associated with each agent is also an expectation; see Appendix \ref{appendix:defns-nash-regret} for a  discussion on variants of Nash regret.   

NSW is an axiomatically-supported welfare objective \cite{moulin2004fair}. That is, in contrast to ad-hoc constraints or adjustments, NSW satisfies a collection of fundamental  axioms, including symmetry, independence of unconcerned agents, scale invariance, and the Pigou-Dalton transfer principle \cite{moulin2004fair}. At a high level, the Pigou-Dalton principle ensures that NSW will increase under a policy change that transfers, say, $\delta$ reward from a well-off agent $t$ to an agent $\widetilde{t}$ with lower current value.\footnote{Recall that NSW is defined as the geometric mean of rewards and, hence, a more balanced collection of rewards will have higher NSW.} At the same time, if the relative increase in $\widetilde{t}$'s value is much less than the drop in $t$'s value, then NSW would not favor such a transfer (i.e., it also accommodates for allocative efficiency). The fact that NSW strikes a balance between fairness and economic efficiency is also supported by the observation that it sits between egalitarian and (average) social welfare: the geometric mean is at least as large as the minimum reward and it is also at most the arithmetic mean (the AM-GM inequality).   

In summary, Nash social welfare induces an order among profiles of expected rewards (by considering their geometric means). Profiles with higher NSW are preferred. Our well-justified goal is to develop an algorithm that induces high NSW among the agents. Equivalently, we aim to develop an algorithm that minimizes Nash regret. 

It is relevant here to note the conceptual correspondence with average regret, wherein profiles with higher social welfare are preferred. That is, while average regret is an appropriate primitive for utilitarian concerns, Nash regret is furthermore relevant when fairness is an additional desideratum.  

\noindent
{\bf Our Results and Techniques.} We develop an algorithm that achieves Nash regret of $O \left( \sqrt{\frac{{k \log T}}{T}} \right)$; here, $k$ denotes the number of arms in the bandit instance and $T$ is the given horizon of play (Theorem \ref{thm:nucb} and Theorem \ref{theorem:improvedNashRegret}).  Note that, for any algorithm, the Nash regret is at least as much as its average regret.\footnote{This follows from the AM-GM inequality: The average regret is equal to the difference between the optimal mean, $\mu^*$, and the arithmetic mean of expected rewards. The arithmetic mean is at least the geometric mean, which in turn in considered in Nash regret.} Therefore, the known  $\Omega\left( \sqrt{ \frac{k}{T} } \right)$ lower bound on average regret~\cite{auer2002nonstochastic} holds for Nash regret as well. This observation implies that, up to a log factor, our guarantee matches the best-possible bound, even for average regret.   

We also show that the standard upper confidence bound (UCB) algorithm \cite{lattimore2020bandit} does not achieve any meaningful guarantee for Nash regret (Appendix \ref{appendix:ucb}). This barrier further highlights that Nash regret is a more challenging benchmark than average regret. In fact, it is not obvious if one can obtain any nontrivial guarantee for Nash regret by directly invoking upper bounds known for average (cumulative) regret. For instance, a reduction from Nash regret minimization to average regret minimization, by taking logs of the rewards (i.e., by converting the geometric mean to the arithmetic mean of logarithms), faces the following hurdles: (i) for rewards that are bounded between $0$ and $1$, the log values can be in an arbitrarily large range, and (ii) an additive bound for the logarithms translates back to only a multiplicative guarantee for the underlying rewards. 

Our algorithm (Algorithm \ref{algo:ncb}) builds upon the UCB template with interesting technical insights; see Section \ref{section:algorithm} for a detailed description. The two distinctive features of the algorithm are: (i) it performs uniform exploration for a judiciously chosen number of initial rounds and then (ii) it adds a novel (arm-specific) confidence width term to each arm's empirical mean and selects an arm for which this sum is maximum (see equation (\ref{eq:ncb})). Notably, the confidence width includes the empirical mean as well. These modifications enable us to go beyond standard regret analysis.\footnote{Note that the regret decomposition lemma \cite{lattimore2020bandit}, a mainstay of regret analysis, is not directly applicable for Nash regret.} 

The above-mentioned algorithmic result focusses on settings in which the horizon of play (i.e., the number of rounds) $T$ is given as input. Extending this result, we also establish a Nash regret guarantee for $T$-oblivious settings. In particular, we develop an anytime algorithm with a Nash regret of $O \left( \sqrt{\frac{{k \log T}}{T}} \log T \right)$ (Theorem \ref{thm:anytime-nucb}). This extension entails an interesting use of empirical estimates to identify an appropriate round at which the algorithm can switch out of uniform exploration. \\ 

\noindent
{\bf Additional Related Work.} Given that learning algorithms are increasingly being used to guide socially sensitive decisions, there has been a surge of research aimed at achieving fairness in MAB contexts; see, e.g., \cite{JOSEPH2016,CELIS2019,PATIL2020,BISTRITZ2020,HOSSAIN2021} and references therein. This thread of research has predominantly focused on achieving fairness for the arms. By contrast, the current work establishes fairness (and welfare) guarantees across time.  

The significance of Nash social welfare and its axiomatic foundations \cite{kaneko1979nash,nash1950bargaining} in fair division settings are well established; see \cite{moulin2004fair} for a textbook treatment. Specifically in the context of allocating divisible goods, NSW is known to uphold other important fairness and efficiency criteria \cite{varian1974equity}. In fact, NSW corresponds to the objective function considered in the well-studied convex program of Eisenberg and Gale \cite{eisenberg1959consensus}. NSW is an object of active investigation in discrete fair division literature as well; see, e.g.,~\cite{caragiannis2019unreasonable}. 

%% file: notation.tex
\section{Notation and Preliminaries}
We study the classic (stochastic) multi-armed bandit problem. Here, an online algorithm (decision maker) has sample access to $k$ (unknown) distributions, that are supported on $[0,1]$. The distributions are referred to as arms $i \in \{1,\ldots, k\}$. The algorithm must iteratively select (pull) an arm per round and this process continues for $T \geq 1$ rounds overall. Successive pulls of an arm $i$ yield i.i.d.~rewards from the $i$th distribution. We will, throughout, write $\mu_i \in [0,1]$ to denote the (a priori unknown) mean of the the $i$th arm and let $\mu^*$ be maximum mean, $\mu^* \coloneqq  \max_{i \in [k]} \ \mu_i$. 
Furthermore, given a bandit instance and an algorithm, the random variable $I_t \in [k]$ denotes the arm pulled in round $t \in \{1, \ldots, T\}$. Note that $I_t$ depends on the draws observed before round $t$. 

We address settings in which the rewards are distributed across a population of $T$ agents. Specifically, for each agent $t \in \{1, \ldots, T\}$, the expected reward received is $\E[\mu_{I_t}]$ and, hence, the algorithm induces rewards  $\left\{ \E [\mu_{I_t}] \right\}_{t=1}^T$ across all the $T$ agents. Notably, one can quantify the algorithm's performance by applying a welfare function on these induced rewards. Our focus is on Nash social welfare, which, in the current context, is equal to the geometric mean of the agents' expected rewards: $\left( \prod_{t=1}^T  \E [\mu_{I_t}]  \right)^{1/T}$. Here, the overarching aim of achieving a Nash social welfare as high as possible is quantitatively captured by considering  \emph{Nash regret}, $\NRg_T$; this metric is defined as
\begin{align}
\NRg_T \coloneqq \mu^* - \left( \prod_{t=1}^T  \E [\mu_{I_t}]  \right)^{1/T} \label{eq:def-NRg}
\end{align}
Note that the optimal value of Nash social welfare is $\mu^*$, and our objective is to minimize Nash regret. 

Furthermore, the standard notion of average (cumulative) regret is obtained by assessing the algorithm's performance as the induced social welfare $\frac{1}{T} \sum_{t=1}^T \E [ \mu_{I_t} ]$. Specifically, we write $\Rg_T$ to denote the average regret, $\Rg_T \coloneqq \mu^* - \frac{1}{T} \sum_{t=1}^T \E [ \mu_{I_t} ]$. The AM-GM inequality implies that Nash regret, $\NRg_T$, is a more challenging benchmark than $\Rg_T$; indeed, for our algorithm, the Nash regret is $O \left( \sqrt{\frac{k \log T}{T}} \right)$ and the same guarantee holds for the algorithm's average regret as well.   


%% file: algorithm.tex
\section{The Nash Confidence Bound Algorithm}
\label{section:algorithm}
Our algorithm (Algorithm \ref{algo:ncb}) consists of two phases. Phase {\rm I} performs uniform exploration for $\widetilde{T} \coloneqq 16\sqrt{\frac{k T \log T}{\log k}}$ rounds. In Phase {\rm II}, each arm is assigned a value (see equation (\ref{eq:ncb})) and the algorithm pulls the arm with the highest current value. Based on the observed reward, the values are updated and this phase continues for all the remaining rounds.

We refer to the arm-specific values as the \emph{Nash confidence bounds}, $\NCB_i$-s. For each arm $i \in [k]$, we obtain $\NCB_i$ by adding a `confidence width' to the empirical mean of arm $i$; in particular, $\NCB_i$ depends on the number of times arm $i$ has been sampled so far and rewards experienced for $i$. Formally, for any round $t$ and arm $i \in [k]$, let $n_i \geq 1$ denote the number of times arm $i$ has been pulled before this round.\footnote{Note that $n_i$ is a random variable.} Also, for each $1 \leq s \leq n_i$, random variable $X_{i,s}$ be the observed reward when arm $i$ is pulled the $s$th time. At this point, arm $i$ has empirical mean $\widehat{\mu}_i \coloneqq \frac{1}{n_i} \sum_{s=1}^{n_i} X_{i, s}$ and we define the Nash confidence bound as
\begin{align}
	\NCB_i & \coloneqq  \widehat{\mu_i} + 4 \sqrt{\frac{\widehat{\mu_i} \log T}{n_i}} \label{eq:ncb}
\end{align}

It is relevant to observe that, in contrast to standard UCB (see, e.g., \cite{lattimore2020bandit}), here the confidence width includes the empirical mean (i.e., the additive term has $\widehat{\mu}_i$ under the square-root). This is an important modification that enables us to go beyond standard regret analysis. Furthermore, we note that the Nash regret guarantee of Algorithm \ref{algo:ncb} can be improved by a factor of $\sqrt{\log k}$ (see Theorem \ref{thm:nucb} and Theorem \ref{theorem:improvedNashRegret}). The initial focus on Algorithm \ref{algo:ncb} enables us to highlight the key technical insights for Nash regret. The improved guarantee is detailed in Section \ref{section:modified-NCB}.    

\begin{algorithm}[ht!]
	\caption{Nash Confidence Bound Algorithm}
	\label{algo:ncb}
	\noindent
	\textbf{Input:} Number of arms $k$ and horizon of play $T$. \\
	\vspace{-10pt}
	\begin{algorithmic}[1]
		\STATE Initialize set $S := [k]$ along with empirical means $\widehat{\mu}_i = 0$ and counts $n_i = 0$ for all $i \in [k]$. Also, set $\widetilde{T} \coloneqq 16\sqrt{\frac{k T \log T}{\log k}}$.
		\\ \texttt{Phase {\rm I}}
		\FOR{$t=1$ to $\widetilde{T}$} 
		\STATE Select $I_t$ uniformly at random from $S$. Pull arm $I_t$ and observe reward $X_t$.
		\STATE For arm $I_t$, increment the count $n_{I_t}$ (by one) and update the empirical mean $\widehat{\mu}_{I_t}$.
		\ENDFOR
		\\ \texttt{Phase {\rm II}}
		\FOR{$t=(\widetilde{T}+1)$ to $T$}
		\STATE Pull the arm $I_t$ with the highest Nash confidence bound, i.e., $I_t = \argmax_{i \in [k]} \ \NCB_i$. \label{step:argmax}
		\STATE Observe reward $X_t$ and update $\widehat{\mu}_{I_t}$. 
		\STATE Update the Nash confidence bound for $I_t$  (see equation (\ref{eq:ncb})).
		\ENDFOR
	\end{algorithmic}
\end{algorithm}

The following theorem is the main result of this section and it establishes that Algorithm \ref{algo:ncb} achieves a tight---up to log factors---guarantee for Nash regret.

\begin{theorem}\label{thm:nucb}
For any bandit instance with $k$ arms and given any (moderately large) time horizon $T$, the Nash regret of Algorithm \ref{algo:ncb} satisfies
\begin{align*}
\NRg_T = O \left( \sqrt{ \frac{k\log k \ \log T}{T} } \right).
\end{align*}
\end{theorem}

\subsection{Regret Analysis}
\label{section:regret-analysis}

We first define a ``good'' event $G$ and show that it holds with high probability (Lemma \ref{lemma:good-event}); our Nash regret analysis is based on conditioning on $G$. In particular, we will first define three sub-events $G_1, G_2, G_3$ and set $ G \coloneqq G_1 \cap G_2 \cap G_3$. For specifying these events, write $\widehat{\mu}_{i, s}$ to denote the empirical mean of arm $i$'s rewards, based on the first $s$ samples (of $i$). 
\begin{itemize}[leftmargin=20pt]
\item[$G_1$:] Every arm $i \in [k]$ is sampled at least $\frac{\widetilde{T}}{2 k}$ times in Phase {\rm I},\footnote{Recall that $\widetilde{T} \coloneqq 16\sqrt{\frac{k T \log T}{\log k}}$.} i.e., for each arm $i$ we have $n_i \geq \frac{\widetilde{T}}{2 k}$ at the end of the first phase in Algorithm \ref{algo:ncb}.
\item[$G_2$:] For all arms $i \in [k]$, with $ \mu_i > \frac{6\sqrt{k \log k\log T}}{\sqrt{T}} $, and all sample counts $\frac{\widetilde{T}}{2k} \leq s \leq T$ we have $\left|\mu_{i}-\widehat{\mu}_{i,s} \right| \leq 3 \sqrt{\frac{\mu_{i} \log T}{s}}$.
\item[$G_3$:] For all arms $j \in [k]$, with $ \mu_j \leq \frac{6\sqrt{k\log k\log T}}{\sqrt{T }} $, and all $\frac{\widetilde{T}}{2k} \leq s \leq T$,  we have $\widehat{\mu}_{j,s} \leq \frac{9 \sqrt{k\log k \log T}}{\sqrt{T }} $.
\end{itemize}
Here,\footnote{Note that if, for all arms $i \in [k]$, the means  $\mu_i \leq \frac{6\sqrt{k\log k\log T}}{\sqrt{T}} $, then, by convention, $\prob \{G_2\} = 1$. Similarly, if all the means are sufficiently large, then $\prob\{G_3\} = 1$.} all the events are expressed as in the canonical bandit model (see, e.g., \cite[Chapter 4]{lattimore2020bandit}). In particular, for events $G_2$ and $G_3$, one considers a $k \times T$ reward table that populates $T$ independent samples for each arm $i \in [k]$. All the empirical means are obtained by considering the relevant entries from the table; see Appendix \ref{appendix:defns-nash-regret} for a more detailed description of the associated probability space. Also note that, conceptually, the algorithm gets to see the $(i, s)$th entry in the table only when it samples arm $i$ the $s$th time. 

The lemma below lower bounds the probability of event $G$; its proof is deferred to Appendix \ref{appendix:prob-G}.
\begin{lemma} \label{lemma:good-event}
	$\prob \left\{ G \right\} \geq \left(1-\frac{4}{T}\right)$.
\end{lemma}

Next, we state a useful numeric inequality; for completeness, we provide its proof in Appendix \ref{appendix:numeric-ineq}.

\begin{restatable}{claim}{LemmaNumeric}
\label{lem:binomial}
For all reals $x \in \left[0, \frac{1}{2}\right]$ and all $a \in [0,1]$, we have $(1-x)^{a} \geq  1- 2ax$. 
\end{restatable}

Now, we will show that the following key guarantees (events) hold under the good event $G$: 
\begin{itemize}[leftmargin=20pt]
\item Lemma \ref{lem:emp}: The Nash confidence bound of the optimal arm $i^*$ is at least its true mean, $\mu^*$, throughout Phase {\rm II}.
\item Lemma \ref{lem:bad_arms}: Arms $j$ with sufficiently small means (in particular, $\mu_j \leq \frac{6\sqrt{k \log k \log T}}{\sqrt T}$) are never pulled in Phase {\rm II}. 
\item Lemma \ref{lem:suboptimal_arms}: Arms $i$ that are pulled many times in Phase {\rm II} have means $\mu_i$ close to the optimal $\mu^*$. Hence, such arms $i$ do not significantly increase the Nash regret. 
\end{itemize}
The proofs of these three lemmas are deferred to Appendix \ref{appendix:key-lemmas}. In these results, we will address bandit instances wherein the optimal mean $\mu^* \geq  \frac{32 \sqrt{k \log k\log T}}{\sqrt T}$. Note that in the complementary case (wherein $\mu^* < \frac{32 \sqrt{k\log k\log T}}{\sqrt T}$) the Nash regret directly satisfies the bound stated in Theorem \ref{thm:nucb}.

\begin{restatable}{lemma}{LemmaNCBopt}
\label{lem:emp}
Let $\NCB_{i^*,t}$ be the Nash confidence bound of the optimal arm $i^*$ at round $t$. Assume that the good event $G$ holds and also $\mu^* \geq \frac{32\sqrt{k \log k\log T}}{\sqrt T}$. Then, for all rounds $t > \widetilde{T}$ (i.e., for all rounds in Phase {\rm II}), we have $\NCB_{i^*,t} \geq \mu^*$.
\end{restatable}
\begin{restatable}{lemma}{LemmaBadArms}
\label{lem:bad_arms}
Consider a bandit instance with optimal mean $\mu^* \geq \frac{32 \sqrt{k\log k\log T}}{\sqrt T}$ and assume that the good event $G$ holds. Then, any arm $j$, with mean $\mu_j \leq \frac{6\sqrt{k\log k \log T}}{\sqrt{T}}$, is never pulled in all of Phase {\rm II}.  
\end{restatable}
\begin{restatable}{lemma}{LemmaHighMean}
\label{lem:suboptimal_arms}
Consider a bandit instance with optimal mean $\mu^* \geq \frac{32 \sqrt{k\log k\log T}}{\sqrt T}$ and assume that the good event $G$ holds. Then, for any arm $i$ that is pulled at least once in Phase {\rm II} we have 
\begin{align*}
\mu_i  \geq \mu^* - 8\sqrt{\frac{\mu^* \log T}{T_i  -1}}, 
\end{align*}
where $T_i$ is the total number of times that arm $i$ is pulled in the algorithm. 
\end{restatable}

\subsection{Proof of Theorem \ref{thm:nucb}}
For bandit instances in which the optimal mean $\mu^* \leq  \frac{32 \sqrt{k\log k\log T}}{\sqrt T}$, the theorem holds directly; specifically, the Nash regret $\NRg_T = \mu^* - \left( \prod_{t=1}^T  \E [\mu_{I_t}]  \right)^{1/T} \leq \mu^*$. Therefore, in the remainder of the proof we will address instances wherein  $\mu^* \geq  \frac{32\sqrt{k\log k \log T}}{\sqrt T}$.

The Nash social welfare of the algorithm satisfies\footnote{Recall that $\widetilde{T} \coloneqq 16\sqrt{\frac{k T \log T}{\log k}}$.}
$\left(\prod_{t=1}^{T} \E \left[ \mu_{I_t} \right] \right)^\frac{1}{T} =  \left(\prod_{t=1}^{\widetilde{T}} \E \left[ \mu_{I_t} \right]\right)^\frac{1}{T} \left(\prod_{t=\widetilde{T} + 1}^{T}  \E \left[ \mu_{I_t} \right] \right)^\frac{1}{T}$. In this product, the two terms account for the rewards accrued in the two phases, respectively. Next, we will lower bound these two terms. \\

\noindent
Phase {\rm I}: In each round of the first phase, the algorithm selects an arm uniformly at random. Hence, $\E[\mu_{I_t}] \geq \frac{\mu^*}{k}$, for each round $t \leq \widetilde{T}$. Therefore, for Phase {\rm I} we have 
\begin{align}
\left(\prod_{t=1}^{\widetilde{T}} \E [\mu_{I_t} ] \right)^\frac{1}{T} & \geq \left(\frac{\mu^*}{k}\right)^{\frac{\widetilde{T}}{T}} =  \left(\mu^*\right)^{\frac{\widetilde{T}}{T}} \left(\frac{1}{k}\right)^{\frac{16\sqrt{k \log T}}{\sqrt{T\log k}}} = \left(\mu^*\right)^{\frac{\widetilde{T}}{T}} \left(\frac{1}{2}\right)^{\frac{16\sqrt{k \log T}\log k}{\sqrt{T\log k}}} \nonumber \\
& = \left(\mu^*\right)^{\frac{\widetilde{T}}{T}} \left(1- \frac{1}{2}\right)^{\frac{16\sqrt{k\log k \log T}}{\sqrt{T}}} 
\geq  \left(\mu^*\right)^{\frac{\widetilde{T}}{T}} \left(1-{\frac{16\sqrt{k\log k \log T}}{\sqrt{T}}}\right)       \label{ineq:phaseone}
\end{align}
To obtain the last inequality we note that the exponent $\frac{16\sqrt{k\log k \log T}}{\sqrt{T}}<1$ (for appropriately large $T$) and, hence, the inequality follows from Claim \ref{lem:binomial}. \\

\noindent
Phase {\rm II}: For the second phase, the product of the expected rewards can be bounded as follows 
\begin{align}
\left(\prod_{t=\widetilde{T}+ 1}^{T} \E\left[ \mu_{I_t} \right]\right)^\frac{1}{T} & \geq \E\left[\left( \prod_{t=\widetilde{T}+1}^{T} \mu_{I_t} \right)^\frac{1}{T}\right ] 
\geq \E \left[\left( \prod_{t=\widetilde{T}+1}^{T} \mu_{I_t} \right)^\frac{1}{T} \;\middle|\; G \right]  \prob\{ G \} \label{ineq:interim} 
\end{align}
Here, the first inequality follows from the multivariate Jensen's inequality and the second one is obtained by conditioning on the good event $G$. To bound the expected value in the right-hand-side of inequality (\ref{ineq:interim}), we consider the arms that are pulled at least once in Phase {\rm II}. In particular, with reindexing, let $\{1,2, \ldots, \ell\}$ denote the set of all arms that are pulled at least once in the second phase. Also, let $m_i \geq 1$ denote the number of times arm $i \in [\ell]$ is pulled in Phase {\rm II} and note that $\sum_{i=1}^\ell m_i = T -  \widetilde{T}$. Furthermore, let $T_i$ denote the total number of times any arm $i$ is pulled in the algorithm. Indeed, $(T_i - m_i)$ is the number of times arm $i \in [\ell]$ is pulled in Phase {\rm I}. With this notation, the expected value in the right-hand-side of inequality (\ref{ineq:interim}) can be expressed as $\E \left[\left( \prod_{t=\widetilde{T}}^{T} \ \mu_{I_t} \right)^\frac{1}{T} \;\middle|\; G \right] = \E\left[\left( \prod_{i=1}^{\ell} \mu_{i}^\frac{m_i}{T}  \right)\;\middle|\; G \right] $. Moreover, since we are conditioning on the good event $G$, Lemma  \ref{lem:suboptimal_arms} applies to each arm $i \in [\ell]$. Hence, 	
\begin{align}
\E \left[\left( \prod_{t=\widetilde{T}}^{T} \ \mu_{I_t} \right)^\frac{1}{T} \;\middle|\; G \right] & = \E\left[\left( \prod_{i=1}^{\ell} \mu_{i}^\frac{m_i}{T}  \right)\;\middle|\; G \right] 
\ \geq \E\left[\prod_{i=1}^{\ell}\left(\mu^* - 8\sqrt{\frac{\mu^* \log T}{T_i -1 }} \right)^\frac{m_i}{T} \;\middle|\; G \right]   \tag{Lemma \ref{lem:suboptimal_arms}} \nonumber\\
& = (\mu^*)^{1-\frac{\widetilde{T}}{T}} \E\left[\prod_{i=1}^{\ell}\left(1 - 8\sqrt{\frac{ \log T}{\mu^*(T_i -1)}} \right)^\frac{m_i}{T}	\;\middle|\; G \right] \label{ineq:dunzo}
\end{align}
For the last equality, we use $\sum_{i=1}^\ell m_i = T - \widetilde{T}$. Now, recall that, under event $G$, each arm is pulled at least $\frac{\widetilde{T}}{2k} = \frac{8}{k} \sqrt{\frac{k T \log T}{\log k}}$ times in Phase {\rm I}. Hence, $T_i > \frac{\widetilde{T}}{2k}$ for each arm $i \in [\ell]$. Furthermore, since $\mu^* \geq \frac{32 \sqrt{k\log k\log T}}{\sqrt T}$, we have $8\sqrt{\frac{ \log T}{\mu^*(T_i - 1)}} \leq 8 \sqrt{\frac{1}{256}} = \frac{1}{2}$ for each $i \in [\ell]$. Therefore, we can apply Claim \ref{lem:binomial} to reduce the expected value in inequality (\ref{ineq:dunzo}) as follows 
\begin{align*}
\E \left[\prod_{i=1}^{\ell}\left(1 - 8\sqrt{\frac{ \log T}{\mu^*(T_i -1 )}} \right)^\frac{m_i}{T}\;\middle|\; G \right]  & \geq \E \left[\prod_{i=1}^{\ell}\left(1 - \frac{16 m_i}{T}\sqrt{\frac{ \log T}{\mu^*(T_i -1)}} \right)\;\middle|\; G \right]                    \\ 
& \geq \E \left[\prod_{i=1}^{\ell}\left(1 - \frac{16 }{T}\sqrt{\frac{ m_i \log T }{\mu^*}} \right)\;\middle|\; G \right] & \text{(since $T_i \geq m_i + 1$)}
\end{align*}
We can further simplify the above inequality by noting that $(1-x)(1-y) \geq 1 - x- y$, for all $x,y \geq 0$.
\begingroup
\allowdisplaybreaks
\begin{align}
\E\left[\prod_{i=1}^{\ell}\left(1 - \frac{16 }{T}\sqrt{\frac{ m_i \log T }{\mu^*}} \right)\;\middle|\; G \right] & \geq \E\left[1 -\sum_{i=1}^{\ell}\left(\frac{16 }{T}\sqrt{\frac{ m_i \log T }{\mu^*}} \right)\;\middle|\; G \right] \nonumber \\
& = 1 -\left(\frac{16 }{T}\sqrt{\frac{ \log T }{\mu^*}} \right) \E\left[ \sum_{i=1}^{\ell} \sqrt{m_i}\;\middle|\; G \right] \nonumber \\
& \geq 1 -\left(\frac{16 }{T}\sqrt{\frac{ \log T }{\mu^*}} \right) \E\left[ \sqrt{\ell} \ \sqrt{\sum_{i=1}^\ell m_i} \;\middle|\; G \right] \tag{Cauchy-Schwarz inequality} \nonumber  \\ 
&  \geq  1 -\left(\frac{16 }{T}\sqrt{\frac{ \log T }{\mu^*}} \right) \E\left[ \sqrt{\ell \ T} \;\middle|\; G \right] \nonumber  \tag{since $\sum_i m_i \leq T$} \nonumber \\ 
& = 1 -\left( 16 \sqrt{\frac{ \log T }{\mu^* T }} \right) \E\left[ \sqrt{\ell} \;\middle|\; G \right] \nonumber \\ 
& \geq 1 - \left( 16 \sqrt{\frac{ k \log T }{\mu^* T }} \right)  \label{ineq:limbo} 
\end{align}
\endgroup

Here, the final inequality holds since $\ell \leq k$. Using (\ref{ineq:limbo}), along with inequalities (\ref{ineq:interim}), and (\ref{ineq:dunzo}), we obtain for Phase {\rm II}:
\begin{align}
\left(\prod_{t=\widetilde{T} + 1}^{T} \E\left[ \mu_{I_t} \right]\right)^\frac{1}{T} \geq (\mu^*)^{1-\frac{\widetilde{T}}{T}} \left ( 1 -16\sqrt{\frac{ k \log T }{\mu^*T}}  \right) \prob \{G\} \label{ineq:toomany}
\end{align}
Inequalities (\ref{ineq:toomany}) and (\ref{ineq:phaseone}) provide relevant bounds for Phase {\rm II} and Phase {\rm I}, respectively. Hence, for the Nash social welfare of the algorithm we have 
\begin{align*}
\left(\prod_{t=1}^{T} \E \left[ \mu_{I_t} \right] \right)^\frac{1}{T} & \geq \mu^*\left(1-{\frac{16\sqrt{k \log k \log T}}{\sqrt{T}}}\right) \left ( 1 -16\sqrt{\frac{ k \log T }{\mu^*T} } \right) \prob\{ G \}  \\
& \geq \mu^*\left(1-{\frac{16\sqrt{k \log k \log T}}{\sqrt{T}}}\right) \left ( 1 -16\sqrt{\frac{ k \log T }{\mu^*T} } \right) \left( 1 - \frac{4}{T} \right) \tag{via  Lemma \ref{lemma:good-event}} \\
& \geq \mu^* \left( 1 -\frac{ 32\sqrt{k \log k\log T} }{\sqrt{\mu^*T}} \right)\left(1-\frac{4}{T}\right)  \\
& \geq \mu^* - \frac{ 32\sqrt{\mu^* k\log k\log T} }{\sqrt{T}} - \frac{4\mu^*}{T}  \\
& \geq \mu^* - \frac{ 32\sqrt{ k\log k\log T} }{\sqrt{T}} - \frac{4}{T} \tag{since $\mu^* \leq 1$}
\end{align*}

Therefore, the Nash regret of the algorithm satisfies  
$\NRg_T = \mu^* - \left(\prod_{t=1}^{T} \E \left[ \mu_{I_t} \right] \right)^\frac{1}{T} \leq \frac{ 32\sqrt{ k\log k\log T} }{\sqrt{T}} + \frac{4}{T}$. Overall, we get that $\NRg_T = O \left( \sqrt{\frac{k\log k\log T}{T} } \right)$. The theorem stands proved. 

\noindent
\paragraph{Remark 1.} Algorithm \ref{algo:ncb} is different from standard UCB, in terms of both design and analysis. For instance, here the empirical means appear in the confidence width and impact the concentration bounds utilized in the analysis. 

\noindent
\paragraph{Remark 2.} As mentioned previously, the Nash regret guarantee obtained in Theorem \ref{thm:nucb} can be improved by a factor of $\sqrt{\log k}$. To highlight the key technical insights, in Algorithm \ref{algo:ncb} we fixed the number of rounds in Phase {\rm I} (to $\widetilde{T}$). However, with an adaptive approach, one can obtain a Nash regret of $O \left( \sqrt{\frac{k \log T}{T} } \right)$, as stated in Theorem \ref{theorem:improvedNashRegret} (Section \ref{section:improved-theorem}). A description of the modified algorithm (Algorithm \ref{algo:ncb:modified}) and the proof of Theorem \ref{theorem:improvedNashRegret} appear in Section \ref{section:modified-NCB}.

%% file: improved+anytime-algorithm.tex
\section{Improved and Anytime Guarantees for Nash Regret}
\label{section:anytime}
This section provides an improved (over Theorem \ref{thm:nucb}) Nash regret guarantee for settings in which the horizon of play $T$ is known in advance. Furthermore, here we also develop a Nash regret minimization algorithm for settings in which the horizon of play $T$ is not known in advance. This anytime algorithm (Algorithm \ref{algo:ncb:anytime} in Section \ref{section:final-anytime}) builds upon the standard doubling trick. The algorithm starts with a guess for the time horizon, i.e., a {\em window} of length $W \in \mathbb{Z}_+$. Then, for $W$ rounds it either (i) performs uniform exploration, with probability $\frac{1}{W^2}$, or (ii) invokes Algorithm \ref{algo:ncb:modified} as a subroutine (with the remaining probability $\left( 1- \frac{1}{W^2}\right)$). This execution for $W$ rounds completes one {\em epoch} of Algorithm \ref{algo:ncb:anytime}. In the subsequent epochs, the algorithm doubles the window length and repeats the same procedure till the end of the time horizon, i.e., till a stop signal is received. 

In Section \ref{section:modified-NCB}, we detail Algorithm \ref{algo:ncb:modified}, which is called as a subroutine in our anytime algorithm and it takes as input a (guess) window length $W$. We will also prove that if Algorithm \ref{algo:ncb:modified} is in fact executed with the actual horizon of play (i.e., executed with $W = T$), then it achieves a Nash regret of  
$O \left( \sqrt{ \frac{k  \log T}{T} } \right)$ (Theorem \ref{theorem:improvedNashRegret}); this provides the improvement mentioned above.

\subsection{Modified Nash Confidence Bound Algorithm}
\label{section:modified-NCB}

Algorithm \ref{algo:ncb:modified} consists of two phases: 
\begin{itemize}[leftmargin=20pt]
\item In Phase \textcal{1}, the algorithm performs uniform exploration until the sum of rewards  for any arm $i$ exceeds a certain threshold.\footnote{Note that this is in contrast to Algorithm \ref{algo:ncb}, in which uniform exploration was performed for a fixed number of rounds.} Specifically, with $n_i$ denoting the number of times an arm $i$ has been pulled so far and $X_{i,s}$ denoting the reward observed for arm $i$ when it is pulled the $s$th time, the exploration continues as long as $\max_i \sum_{s=1}^{n_i} X_{i,s} \leq 420c^2\log W$; here $c$ is an absolute constant. Note that this stopping criterion is equivalent to $n_i \ \widehat{\mu}_i \leq 420c^2\log W$, where $\widehat{\mu}_i$ is the empirical mean for arm $i$.
\item In Phase \textcal{2}, the algorithm associates with each arm $i$ the following Nash confined bound value, $\overline{\NCB}_i$, and selects the arm for which that value is the maximized.\footnote{$\overline{\NCB}_i$ differs from $\NCB_i$ (see equation (\ref{eq:ncb})) in terms of constants. Specifically, the parameter $c$ is an absolute constant and is fixed in the algorithm.} 
\begin{align}
    \overline{\NCB}_i & \coloneqq  \widehat{\mu_i} + 2c \sqrt{\frac{2\widehat{\mu_i} \log W}{n_i}} \label{eqn:ncb:anytime}
\end{align}
\end{itemize}

\begin{algorithm}[ht!]
    \caption{Modified NCB }
    \label{algo:ncb:modified}
    \noindent
    \textbf{Input:} Number of arms $k$ and time window $W$.\\
    \vspace{-10pt}
    \begin{algorithmic}[1]
        \STATE Initialize empirical means $\widehat{\mu}_i = 0$ and counts $n_i = 0$ for all $i \in [k]$.
        \STATE Initialize round index $t=1$ and set parameter $c = 3$.
        \\ \texttt{Phase \textcal{1}}
        \WHILE{ $\max_{i}{n_i \ \widehat{\mu}_i} \leq 420c^2\log W$ and $t \leq W$} \label{step:PhaseOneAlgTwo}
        \STATE Select $I_t$ uniformly at random from $[k]$. Pull arm $I_t$ and observe reward $X_t$.
        \STATE For arm $I_t$, increment the count $n_{I_t}$ (by one) and update the empirical mean $\widehat{\mu}_{I_t}$.
        \STATE Update $t\gets t + 1$. 
        \ENDWHILE
        \\ \texttt{Phase \textcal{2}}
        \WHILE{ $t \leq W$}
        \STATE Pull the arm $I_t$ with the highest Nash confidence bound, i.e., $I_t = \argmax_{i \in [k]} \ \overline{\NCB}_i$.
        \STATE Observe reward $X_t$ and update the Nash confidence bound for $I_t$ (see equation (\ref{eqn:ncb:anytime})).
        \STATE Update $t\gets t + 1$.
        \ENDWHILE
    \end{algorithmic}
\end{algorithm}

Recall that Algorithm \ref{algo:ncb:modified} is called as a subroutine by our anytime algorithm (Algorithm \ref{algo:ncb:anytime}) with a time window (guess) $W$ as input. For the purposes of analysis (see Section \ref{section:final-anytime} for details), it suffices to obtain guarantees for Algorithm \ref{algo:ncb:modified} when $W$ is at least $\sqrt{T}$. Hence, this section analyzes the algorithm with the assumption\footnote{We will also assume that the optimal mean $\mu^*$ is sufficiently greater than $\frac{1}{\sqrt{T}}$. In the complementary case, the stated Nash regret bound follows directly.}  that $\sqrt{T}\leq W \leq T$.

We first define a ``good'' event $E$ and show that it holds with high probability; our analysis is based on conditioning on $E$. In particular, we will first define three sub-events $E_{1}, E_{2}, E_{3}$ and set $ E \coloneqq E_{1} \cap E_{2} \cap E_{3}$. For specifying these events, write $\widehat{\mu}_{i, s}$ to denote the empirical mean of arm $i$'s rewards, based on the first $s$ samples (of $i$). Also, define 
\begin{align}
\Sample \coloneqq \frac{c^2 \log T}{ \mu^*} \label{eqn:tildeW}
\end{align}

\begin{itemize}[leftmargin=20pt]
    \itemsep0em
    \item[$E_{1}$:] For any number of rounds $r \geq128 k  \Sample $ and any arm $i\in [k]$, during the first $r$ rounds of uniform sampling, arm $i$ is sampled at least $\frac{r}{2k}$ times and at most $\frac{3r}{2k}$ times.
    \item[$E_{2}$:] For all arms $i \in [k]$, with $ \mu_i > \frac{\mu^*}{64} $, and all sample counts $64 \Sample \leq s \leq T$ we have $\left|\mu_{i}-\widehat{\mu}_{i,s} \right| \leq c \sqrt{\frac{\mu_{i} \log T}{s}}$.
    \item[$E_{3}$:] For all arms $j \in [k]$, with $ \mu_j \leq \frac{\mu^*}{64}$, and all sample counts $64 \Sample \leq s \leq T$,  we have $\widehat{\mu}_{j,s} < \frac{\mu^*}{32}$.
\end{itemize}

Note that these events address a single execution of Algorithm \ref{algo:ncb:modified} and are expressed in the canonical bandit model \cite{lattimore2020bandit}. Furthermore, they are expressed using the overall horizon of play $T$. This, in particular, ensures that, irrespective of $W$, they are well specified.  

We first obtain a probability bound for the event $E$; the proof of the following lemma is deferred to Appendix \ref{appendix:proof-of-E}. 

\begin{lemma}
\label{lemma:modifiedgoodeventpr}
$\prob \left\{ E \right\} \geq 1 - \frac{4}{T}$.
\end{lemma}

The next lemma shows that, under event $E$,  the total observed reward for any arm is low until certain number of samples. In the final analysis, this result will enable us to bound (under event $E$) the number of rounds in Phase \textcal{1} of Algorithm \ref{algo:ncb:modified}. The proofs of Lemmas \ref{tau:lower}, \ref{tau:upper}, and \ref{tau:anytime} appear in Appendix \ref{appendix:modifiedncb-supporting-lemmas}.
 
\begin{restatable}{lemma}{LemmaTauLower}
\label{tau:lower}
 Under the event $E$, for any arm $i$ and any sample count $n \leq 192  \Sample$, we have $n \ \widehat{\mu}_{i,n} <  210c^2\log T$.
\end{restatable}

Recall that $i^*$ denotes the optimal arm, i.e., $i^*=\arg\max_{i\in[k]}\mu_i$. The following lemma shows that, under event $E$ and after certain number of samples, the total observed reward for $i^*$ is sufficiently large. 
\begin{restatable}{lemma}{LemmaTauUpper}
\label{tau:upper}
Under the event $E$, for any sample count $n \geq 484  \Sample$, we have $n \ \widehat{\mu}_{i^*, n} \geq  462c^2\log T$.
\end{restatable}

The lemma below will help us in analyzing the termination of the first while-loop (Line \ref{step:PhaseOneAlgTwo}) of Algorithm \ref{algo:ncb:modified}. Also, recall that in this section we analyze Algorithm \ref{algo:ncb:modified} with the assumption that $\sqrt{T}\leq W \leq T$.  

\begin{restatable}{lemma}{LemmaTauAnytime}
\label{tau:anytime}
Assume that $\sqrt{T} \leq W \leq T$. Also, let random variable $\tau$ denote the number of rounds of uniform sampling at which the sum of observed rewards for any arm exceeds $420 c^2\log W$ (i.e., only after $\tau$ rounds of uniform sampling we have $\max_i \ n_i \widehat{\mu}_i >420 c^2\log W$). Then, under event $E$, the following bounds hold 
\begin{align*}
128 \ k \Sample \leq \tau \leq 968 \ k  \Sample. 
\end{align*} 
\end{restatable}

As mentioned previously, the events $E_1$, $E_2$, and $E_3$ are defined under the canonical bandit model. Hence, Lemmas \ref{tau:lower}, \ref{tau:upper}, and \ref{tau:anytime} also conform to this setup. 

Next, we will show that the following key guarantees (events) hold under the good event $E$: 
\begin{itemize}[leftmargin=20pt] 
\item Lemma \ref{lem:emp:anytime}: The Nash confidence bound of the optimal arm $i^*$ is at least its true mean, $\mu^*$, throughout Phase \textcal{2} of Algorithm \ref{algo:ncb:modified}.
\item Lemma \ref{lem:bad_arms:anytime}: Arms $j$ with sufficiently small means (in particular, $\mu_j \leq \frac{\mu^*}{64}$) are never pulled in Phase \textcal{2}. 
\item Lemma \ref{lem:suboptimal_arms:anytime}: Arms $i$ that are pulled many times in Phase \textcal{2} have means $\mu_i$ close to the optimal $\mu^*$. Hence, such arms $i$ do not significantly increase the Nash regret.
\end{itemize}

\begin{restatable}{lemma}{LemmaNCBopt:anytime}
    \label{lem:emp:anytime}
    Let $\overline{\NCB}_{i^*,t}$ be the Nash confidence bound of the optimal arm $i^*$ at any round $t$ in Phase \textcal{2}. Assume that the good event $E$ holds and $\sqrt{T} \leq W \leq T$. Then, we have $\overline{\NCB}_{i^*,t} \geq \mu^*$.
\end{restatable}
\begin{proof}
    Fix any round $t$ in Phase \textcal{2} and write $n_{i^*}$ to denote the number of times the optimal arm $i^*$ has been pulled before that round. Also, let $\widehat{\mu}^*$ denote the empirical mean of arm $i^*$ at round $t$. Hence, by definition, at this round the Nash confidence bound ${\overline{\NCB}_{i^*,t}} \coloneqq  \widehat{\mu}^* + 2c \sqrt{\frac{2\widehat{\mu}^* \log W}{n_{i^*}}}$. 
    
Since event $E$ holds, Lemma \ref{tau:anytime} implies that Algorithm \ref{algo:ncb:modified} must have executed at least $128 k \Sample$ rounds in Phase \textcal{1} (before switching to Phase \textcal{2}): the termination condition of the first while-loop (Line \ref{step:PhaseOneAlgTwo}) is realized only after $128 k \Sample$ rounds.

This lower bound on uniform exploration and event $E_1$ give us $n_{i^*} \geq 64 \Sample$. Therefore, the product $n_{i^*} \mu^* \geq 64c^2\log T$. This inequality enables us to express the empirical mean of the optimal arm as follows 
\begin{align}
\widehat{\mu}^* & \geq \mu^* - c \sqrt{\frac{\mu^* \log T}{n_{i^*}}} \tag{since $n_{i^*} \geq 64 \Sample$ and event $E_2$ holds}    \\
                        & = \mu^* - c \mu^* \sqrt{\frac{\log T}{\mu^* n_{i^*}}} \nonumber \\
                        & \geq \mu^* - c \mu^* \sqrt{\frac{1}{64c^2}}                  \tag{since  $\mu^* \ n_{i^*} \geq 64c^2\log T$}   \\
                        & = \frac{7}{8} \mu^*. \nonumber
    \end{align}
    Therefore,
    \begin{align*}
        {\overline{\NCB}_{i^*,t}} & = \widehat{\mu}^* + 2c \sqrt{\frac{2\widehat{\mu}^* \log W}{n_{i^*}}}                                  \\
	& \geq \widehat{\mu}^* + 2c \sqrt{\frac{\widehat{\mu}^* \log T}{n_{i^*}}} \tag{since $2\log W\geq \log T$}\\
                       & \geq  \mu^* - c \sqrt{\frac{\mu^* \log T}{n_{i^*}}} + 2c \sqrt{\frac{\widehat{\mu}^* \log T}{n_{i^*}}} \tag{due to the event $E_{2}$}\\
                       & \geq  \mu^* - c \sqrt{\frac{\mu^* \log T}{n_{i^*}}}  + 2c \sqrt{\frac{7\mu^* \log T}{8n_{i^*}}}     \tag{since $ \widehat{\mu}^*\geq  \frac{7}{8} \mu^*$}   \\
                       & \geq \mu^*
    \end{align*}
    The lemma stands proved.
\end{proof}

\begin{restatable}{lemma}{LemmaBadArms:anytime}
\label{lem:bad_arms:anytime}
Assume that the good event $E$ holds and $\sqrt{T} \leq W \leq T$. Then, any arm $j$, with mean $\mu_j \leq \frac{\mu^*}{64}$, is never pulled in all of Phase \textcal{2}.
\end{restatable}
\begin{proof}
Fix any arm $j$ with mean $\mu_j \leq \frac{\mu^*}{64}$. Let $r_j$ denote the number of times arm $j$ is pulled in Phase \textcal{1}. 

We will fist show that $r_j \geq 64 \Sample$. Since event $E$ holds, Lemma \ref{tau:anytime} ensures that Algorithm \ref{algo:ncb:modified} must have executed at least $128 k \Sample$ rounds in Phase \textcal{1} (before switching to Phase \textcal{2}). This lower bound on uniform exploration and event $E_1$ give us $r_j \geq 64 \Sample$.

Furthermore, event $E_3$ and the fact that $r_j \geq 64 \Sample$ imply that (throughout Phase \textcal{2}) the empirical mean of arm $j$ satisfies $\widehat{\mu}_j \leq \frac{\mu^*}{32}$.  

For any round $t$ in Phase \textcal{2}, write $\overline{\NCB}_{j,t}$ to denote the Nash confidence bound of arm $j$ at round $t$. Below we show that the $\overline{\NCB}_{j, t} $ is strictly less than $\overline{\NCB}_{i^*,t}$ and, hence, arm $j$ is not even pulled once in all of Phase \textcal{2}.
\begingroup
\allowdisplaybreaks
    \begin{align*}
        \overline{\NCB}_{j,t} & = \widehat{\mu}_j + 2c\sqrt{\frac{2\widehat{\mu}_j \log W}{r_j}}                          \\
                   & \leq \widehat{\mu}_j + 2c\sqrt{\frac{2\widehat{\mu}_j \log T}{r_j}}             \tag{since $\log W\leq \log T$}          \\
                   & \leq \frac{\mu^*}{32} + 2c\sqrt{\frac{\mu^* \log T}{16 r_j} }      \tag{since $\widehat{\mu}_j \leq \frac{\mu^*}{32}$}                       \\
                   & \leq \frac{\mu^*}{32} + \frac{c}{2}\sqrt{\frac{\mu^* \log T}{64 \Sample} } \tag{since $r_j \geq 64 \Sample$}\\
                   & \leq \frac{\mu^*}{32} + \frac{\mu^*}{16}                                                  \\
                   & = \frac{3}{32} \mu^* \\
                   & < \overline{\NCB}_{i^*,t} \tag{via Lemma \ref{lem:emp:anytime}}
    \end{align*}
\endgroup
    This completes the proof of the lemma.
\end{proof}

\begin{restatable}{lemma}{LemmaHighMean:anytime}
\label{lem:suboptimal_arms:anytime}
Assume that the good event $E$ holds and $\sqrt{T} \leq W \leq T$. Then, for any arm $i$ that is pulled at least once in Phase \textcal{2} we have
$\mu_i  \geq \mu^* - 4c\sqrt{\frac{\mu^* \log T}{T_i  -1}}$,
where $T_i$ is the total number of times that arm $i$ is pulled in Algorithm \ref{algo:ncb:modified}.
\end{restatable}
\begin{proof}
Fix any arm $i$ that is pulled at least once in Phase \textcal{2}. When arm $i$ was pulled the $T_i$th time during Phase \textcal{2}, it must have had the maximum Nash confidence bound value; in particular, at that round $\overline{\NCB}_i \geq \overline{\NCB}_{i^*} \geq \mu^*$; the last inequality follows from Lemma  \ref{lem:emp:anytime}. Therefore,  we have
\begin{align}
\widehat{\mu}_i + 2c\sqrt{\frac{2\widehat{\mu}_i \log T}{T_i -1 }} \geq \mu^* \label{ineq:hatmu:anytime}
\end{align}
Here, $\widehat{\mu}_i$ denotes the empirical mean of arm $i$ at this point. 

As argued in the proof of Lemmas \ref{lem:emp:anytime} and \ref{lem:bad_arms:anytime}, event $E$ ensures that any arm that is pulled in Phase \textcal{2} is sampled at least $64 \Sample$ times in Phase \textcal{1}. Hence, in particular, we have $T_i > 64 \Sample$. In addition, since arm $i$ is pulled at least once in Phase \textcal{2}, Lemma \ref{lem:bad_arms:anytime} implies that $\mu_i > \frac{\mu^*}{64}$. 

Now, complementing inequality (\ref{ineq:hatmu:anytime}), we will now upper bound the empirical mean $\widehat{\mu}_i$ in terms of $\mu^*$. Specifically, 
\begin{align}
\widehat{\mu_i} & \leq  \mu_i + c\sqrt{\frac{\mu_i \log T}{T_i -1}}                \tag{since $\mu_i >\frac{\mu^*}{64}$ and event $E_{2}$ holds} \nonumber                                                    \\
& \leq \mu^* +c\sqrt{\frac{\mu^* \log T}{ 64 \Sample}} \tag{since $T_i> 64 \Sample$ and $\mu_i \leq \mu^*$}        \\
& = \mu^* + \frac{\mu^*}{8}  \tag{since $\Sample = \frac{c^2 \log T}{\mu^*}$}\\
& = \frac{9}{8} \mu^*  \label{ineq:hatmuo:anytime}
\end{align}

    Inequalities (\ref{ineq:hatmu:anytime}) and (\ref{ineq:hatmuo:anytime}) give us
    \begin{align*}
         & \mu^* \leq \widehat{\mu}_i + 2c\sqrt{\frac{9\mu^* \log T}{4(T_i -1)}}                                                                      \\
         & \leq \mu_i + c\sqrt{\frac{\mu_i \log T}{T_i -1 }} + 3c\sqrt{\frac{\mu^* \log T}{T_i -1}}   \tag{via event $E_{2}$}                           \\
         & \leq \mu_i + c\sqrt{\frac{\mu^* \log T}{T_i  -1}} + 3c\sqrt{\frac{\mu^* \log T}{T_i -1}}                    \tag{since $\mu_i \leq \mu^*$} \\
         & \leq \mu_i +  4c\sqrt{\frac{\mu^* \log T}{T_i -1 }}.
    \end{align*}
    This completes the proof of the lemma.
\end{proof}

\hfill \\

Using the above-mentioned lemmas, we will now establish an essential bound on the Nash social welfare of Algorithm \ref{algo:ncb:modified}.

\begin{lemma}
\label{lem:modified_ncb}
Consider a bandit instance with optimal mean $\mu^* \geq  \frac{512\sqrt{k\log T}}{\sqrt T}$ and assume that $\sqrt{T}\leq W\leq T$. Then, for any $w\leq W$, we have 
\begin{align*}
\left(\prod_{t=1}^{w} \E \left[ \mu_{I_t} \right] \right)^\frac{1}{T} & \geq (\mu^*)^{\frac{w}{T}} \left(1-1000c\sqrt{\frac{ k \log T }{\mu^*T} }\right).
\end{align*}
\end{lemma}
\begin{proof}
First, for the expected rewards $\E \left[ \mu_{I_t} \right]$  (of Algorithm \ref{algo:ncb:modified}), we will derive a lower bound that holds for all rounds $t$.  
In Algorithm \ref{algo:ncb:modified}, for any round $t$ (i.e., $t \leq W$), write $p_t$ to denote the probability that the algorithm is in Phase \textcal{1} and, hence, with probability $(1-p_t)$ the algorithm is in Phase \textcal{2}. That is, with probability $p_t$ the algorithm is selecting an arm uniformly at random and receiving an expected reward of at least $\frac{\mu^*}{k}$. Complementarily, if the algorithm is in Phase \textcal{2} at round $t$, then its expected reward is at least $\frac{\mu^*}{64}$ (Lemma \ref{lem:bad_arms:anytime}). These observations give us 
\begin{align}
 \E[\mu_{I_t}]&\geq \E \left[\mu_{I_t}|E \right] \ \mathbb{P}\{E\} \nonumber \\
& \geq  \E \left[\mu_{I_t}|E \right] \left(1-\frac{4}{T}\right) \tag{Lemma \ref{lemma:modifiedgoodeventpr}}\\
 &=  \left(1-\frac{4}{T}\right)\left(p_t  \frac{\mu^*}{k} + (1-p_t)  \frac{\mu^*}{64}\right) \nonumber \\
 &\geq  \left(1-\frac{4}{T}\right)\left(\frac{\mu^*}{64k}\right) \label{ineq:UniformExpLB}
\end{align}

Towards a case analysis, define threshold $\overline{T} \coloneqq {968 \ k \Sample}$. We observe that Lemma \ref{tau:anytime} ensures that by the $\overline{T}$th round Algorithm \ref{algo:ncb:modified} would have completed Phase \textcal{1}; in particular, the termination condition of the first while-loop (Line \ref{step:PhaseOneAlgTwo}) in the algorithm would be met by the $\overline{T}$th round. Also, note that, under the lemma assumption on $\mu^*$ and for an appropriately large $T$ we have
\begin{align}
\frac{\overline{T} \ \log \left(64k \right)}{T} =  \frac{968 \ k \Sample \ \log \left( 64k\right) }{T} = \frac{968 \ k c^2 \log T \ \log \left( 64k\right) }{\mu^* T} \leq \frac{968 c^2 \log (64k) \sqrt{k \log T}}{512 \sqrt{T}} \leq 1 \label{ineq:overlineTexp}
\end{align} 

We establish the lemma considering two complementary and exhaustive cases based on the given round index $w$: \\ 
\noindent
{\it Case 1:} $w \leq \overline T$, and \\
\noindent
{\it Case 2:} $w > \overline{T}$. \\

\noindent
For {\it Case 1} ($w \leq \overline T$), using inequality (\ref{ineq:UniformExpLB}) we obtain 
\begingroup
\allowdisplaybreaks
\begin{align}
    \left(\prod_{t=1}^{w} \E [\mu_{I_t} ] \right)^\frac{1}{T}
     & \geq \left(1-\frac{4}{T}\right)^{\frac{w}{T}}\left(\frac{\mu^*}{64k}\right)^{\frac{w}{T}}\nonumber                                 \\
     & \geq \left(1-\frac{4}{T}\right) \left(\mu^*\right)^{\frac{w}{T}} \left(\frac{1}{64k}\right)^{\frac{w}{T}} \nonumber                \\
     & = \left(1-\frac{4}{T}\right)\left(\mu^*\right)^{\frac{w}{T}} \left(\frac{1}{2}\right)^{\frac{w \  \log (64k)}{T}} \nonumber      \\
     & \geq \left(1-\frac{4}{T}\right)\left(\mu^*\right)^{\frac{w}{T}} \left(\frac{1}{2}\right)^{\frac{\overline{T} \  \log (64k)}{T}}  \tag{since $w \leq \overline{T}$} \\
     & = \left(1-\frac{4}{T}\right)\left(\mu^*\right)^{\frac{w}{T}} \left(1- \frac{1}{2}\right)^{\frac{\overline{T} \ \log (64k)}{T}} \nonumber   \\
     & \geq  \left(\mu^*\right)^{\frac{w}{T}} \left(1-{\frac{\overline{T}  \log (64k)}{T}}\right)\left(1-\frac{4}{T}\right) \tag{via inequality (\ref{ineq:overlineTexp}) and Claim \ref{lem:binomial}}\\
       & 	\geq \left(\mu^*\right)^{\frac{w}{T}} \left(1-{\frac{\overline{T}  \log (64k)}{T}}-\frac{4}{T}\right)\nonumber\\ 
       & = \left(\mu^*\right)^{\frac{w}{T}} \left(1-{\frac{968c^2 k \log T \  \log (64k)}{ \mu^* T}}-\frac{4}{T}\right) \nonumber \\
  & = \left(\mu^*\right)^{\frac{w}{T}} \left(1- \frac{968c \sqrt{k \log T}}{\sqrt{\mu^* T}} \cdot \frac{c \log{(64k) \sqrt{k \log T }}}{\sqrt{\mu^* T}} -\frac{4}{T}\right) \nonumber \\
      & \geq (\mu^*)^{\frac{w}{T}}\left(1-1000c\sqrt{\frac{ k \log T }{\mu^*T} }\right) \label{ineq:CaseOneW}
    \end{align}
\endgroup    
  
The last inequality follows from the fact that $\frac{c \log{(64k) \sqrt{k \log T }}}{\sqrt{\mu^* T}}   \leq 1$ for an appropriately large $T$; recall that $\mu^* \geq \frac{512\sqrt{k\log T}}{\sqrt T}$. \\



\noindent
For {\em Case 2} ($w>\overline{T}$), we partition the Nash social welfare into two terms: 
\begin{align}
    \left(\prod_{t=1}^{w} \E \left[ \mu_{I_t} \right] \right)^\frac{1}{T} =  \left(\prod_{t=1}^{\overline{T}} \E \left[ \mu_{I_t} \right]\right)^\frac{1}{T} \left(\prod_{t=\overline{T} + 1}^{w}  \E \left[ \mu_{I_t} \right] \right)^\frac{1}{T} \label{eqn:splitW}
\end{align} 

In this product, the two terms account for the rewards accrued in rounds $t \leq \overline{T}$ and in rounds $\overline{T} < t \leq w$, respectively. We will now lower bound these two terms separately.

The first term in the right-hand side of equation (\ref{eqn:splitW}) can be bounded as follows
\begingroup
\allowdisplaybreaks
\begin{align}
    \left(\prod_{t=1}^{\overline{T}} \E [\mu_{I_t} ] \right)^\frac{1}{T} 
     & \geq \left(1-\frac{4}{T}\right)^{\frac{\overline{T}}{T}}\left(\frac{\mu^*}{64k}\right)^{\frac{\overline{T}}{T}} \tag{via inequality (\ref{ineq:UniformExpLB})}                                 \\
     & \geq \left(1-\frac{4}{T}\right) \left(\mu^*\right)^{\frac{\overline{T}}{T}} \left(\frac{1}{64k}\right)^{\frac{\overline{T}}{T}} \nonumber                \\
     & = \left(1-\frac{4}{T}\right)\left(\mu^*\right)^{\frac{\overline{T}}{T}} \left(\frac{1}{2}\right)^{\frac{\overline{T}  \log (64k)}{T}} \nonumber      \\
     & = \left(1-\frac{4}{T}\right)\left(\mu^*\right)^{\frac{\overline{T}}{T}} \left(1- \frac{1}{2}\right)^{\frac{\overline{T}  \log (64k)}{T}} \nonumber   \\
     & \geq  \left(\mu^*\right)^{\frac{\overline{T}}{T}} \left(1-{\frac{\overline{T}  \log (64k)}{T}}\right)\left(1-\frac{4}{T}\right)    \label{ineq:phaseone:anytime}
\end{align}
\endgroup
For establishing the last inequality we note that the exponent ${\frac{\overline{T}  \log (64k)}{T}} \leq 1$ (see inequality (\ref{ineq:overlineTexp})) and apply Claim \ref{lem:binomial}.

For the second term in the right-hand side of equation (\ref{eqn:splitW}), we have 
\begin{align}
    \left(\prod_{t=\overline{T}+ 1}^{w} \E\left[ \mu_{I_t} \right]\right)^\frac{1}{T}
     & \geq \E\left[\left( \prod_{t=\overline{T}+1}^{w} \mu_{I_t} \right)^\frac{1}{T}\right ] \tag{Multivariate Jensen's inequality} \nonumber               \\
     & \geq \E \left[\left( \prod_{t=\overline{T}+1}^{w} \mu_{I_t} \right)^\frac{1}{T} \;\middle|\; E \right]  \prob\{ E \} \label{ineq:interim:anytime}
\end{align}
As mentioned previously, Lemma \ref{tau:anytime} ensures that by the $\overline{T}$th round Algorithm \ref{algo:ncb:modified} would have completed Phase \textcal{1}. Hence, any round $t > \overline{T}$ falls under Phase \textcal{2}. Now, to bound the expected value in the right-hand-side of inequality (\ref{ineq:interim:anytime}), we consider the arms that are pulled at least once after the first $\overline{T}$ rounds. In particular, with reindexing, let $\{1,2, \ldots, \ell\}$ denote the set of all arms that are pulled at least once after the first $\overline{T}$ rounds; note that these $\ell$ arms are in fact pulled in Phase \textcal{2}. Also, let $m_i \geq 1$ denote the number of times arm $i \in [\ell]$ is pulled after the first $\overline{T}$ rounds and note that $\sum_{i=1}^\ell m_i = w -  \overline{T}$. Furthermore, let $T_i$ denote the total number of times any arm $i$ is pulled in the algorithm. Indeed, $(T_i - m_i)$ is the number of times arm $i \in [\ell]$ is pulled during the first $\overline{T}$ rounds. With this notation, the expected value in the right-hand-side of inequality (\ref{ineq:interim:anytime}) can be expressed as $\E \left[\left( \prod_{t=\overline{T}+1}^{T} \ \mu_{I_t} \right)^\frac{1}{T} \;\middle|\; E \right] = \E\left[\left( \prod_{i=1}^{\ell} \mu_{i}^\frac{m_i}{T}  \right)\;\middle|\; E \right] $. Moreover, since we are conditioning on the good event $E$, Lemma  \ref{lem:suboptimal_arms:anytime} applies to each arm $i \in [\ell]$. Hence,
\begin{align}
    \E \left[\left( \prod_{t=\overline{T}+1}^{w} \ \mu_{I_t} \right)^\frac{1}{T} \;\middle|\; E \right] & = \E\left[\left( \prod_{i=1}^{\ell} \mu_{i}^\frac{m_i}{T}  \right)\;\middle|\; E \right] \nonumber \\                                                                                             
    &  \geq \E\left[\prod_{i=1}^{\ell}\left(\mu^* - 4c\sqrt{\frac{\mu^* \log T}{T_i -1 }} \right)^\frac{m_i}{T} \;\middle|\; E \right]   \tag{Lemma \ref{lem:suboptimal_arms:anytime}} \nonumber                                                                                                      \\
                                                                                                           & = (\mu^*)^{\frac{w-\overline{T}}{T}} \E\left[\prod_{i=1}^{\ell}\left(1 - 4c\sqrt{\frac{ \log T}{\mu^*(T_i -1)}} \right)^\frac{m_i}{T}	\;\middle|\; E \right] \label{ineq:dunzo:anytime}
\end{align}
For the last equality, we use $\sum_{i=1}^\ell m_i = w - \overline{T}$. Now under the good event $E$, recall that each arm is pulled at least $64 \Sample$ times during the first $\overline{T}$ rounds. Hence, $T_i > 64 \Sample$, for each arm $i \in [\ell]$, and we have $4c\sqrt{\frac{ \log T}{\mu^*(T_i - 1)}} \leq 4c \sqrt{\frac{\log T}{64c^2\log T}} = \frac{1}{2}$ for each $i \in [\ell]$. Therefore, we can apply Claim \ref{lem:binomial} to reduce the expected value in inequality (\ref{ineq:dunzo:anytime}) as follows
\begin{align*}
    \E \left[\prod_{i=1}^{\ell}\left(1 - 4c\sqrt{\frac{ \log T}{\mu^*(T_i -1 )}} \right)^\frac{m_i}{T}\;\middle|\; E \right] & \geq \E \left[\prod_{i=1}^{\ell}\left(1 - \frac{8c \ m_i}{T}\sqrt{\frac{ \log T}{\mu^*(T_i -1)}} \right)\;\middle|\; E \right]                                     \\
                                                                                                                             & \geq \E \left[\prod_{i=1}^{\ell}\left(1 - \frac{8c}{T}\sqrt{\frac{ m_i \log T }{\mu^*}} \right)\;\middle|\; E \right]             & \text{(since $T_i \geq m_i + 1$)}
\end{align*}
We can further simplify the above inequality by noting that $(1-x)(1-y) \geq 1 - x- y$ for all $x,y \geq 0$.
\begingroup
\allowdisplaybreaks
\begin{align*}
    \E\left[\prod_{i=1}^{\ell}\left(1 - \frac{8c }{T}\sqrt{\frac{ m_i \log T }{\mu^*}} \right)\;\middle|\; E \right] & \geq \E\left[1 -\sum_{i=1}^{\ell}\left(\frac{8c}{T}\sqrt{\frac{ m_i \log T }{\mu^*}} \right)\;\middle|\; E \right]                                                       \\
                                                                                                                     & = 1 -\left(\frac{8c }{T}\sqrt{\frac{ \log T }{\mu^*}} \right) \E\left[ \sum_{i=1}^{\ell} \sqrt{m_i}\;\middle|\; E \right]                                                \\
                                                                                                                     & \geq 1 -\left(\frac{8c}{T}\sqrt{\frac{ \log T }{\mu^*}} \right) \E\left[ \sqrt{\ell} \ \sqrt{\sum_{i=1}^\ell m_i} \;\middle|\; E \right] \tag{Cauchy-Schwarz inequality} \\
                                                                                                                     & \geq 1 -\left(\frac{8c }{T}\sqrt{\frac{ \log T }{\mu^*}} \right) \E\left[ \sqrt{\ell \ T} \;\middle|\; E \right] \tag{since $\sum_i m_i \leq T$}                         \\
                                                                                                                     & = 1 -\left( 8c \sqrt{\frac{ \log T }{\mu^* T }} \right) \E\left[ \sqrt{\ell} \;\middle|\; E \right]                                                                      \\
                                                                                                                     & \geq 1 - \left( 8c\sqrt{\frac{ k \log T }{\mu^* T }} \right) \tag{since $\ell \leq k$}
\end{align*}
\endgroup

Using this bound, along with inequalities (\ref{ineq:interim:anytime}), and (\ref{ineq:dunzo:anytime}), we obtain 
\begin{align}
    \left(\prod_{t=\overline{T} + 1}^{w} \E\left[ \mu_{I_t} \right]\right)^\frac{1}{T} \geq (\mu^*)^{\frac{w-\overline{T}}{T}} \left ( 1 -8c\sqrt{\frac{ k \log T }{\mu^*T}}  \right) \prob \{E\} \label{ineq:toomany:anytime}
\end{align}
Inequalities (\ref{ineq:toomany:anytime}) and (\ref{ineq:phaseone:anytime}) provide relevant bounds for the two terms in equation (\ref{eqn:splitW}), respectively. 
Hence, for the Nash social welfare of the algorithm, we have
\begin{align*}
    \left(\prod_{t=1}^{w} \E \left[ \mu_{I_t} \right] \right)^\frac{1}{T}
     & \geq (\mu^*)^{\frac{w}{T}}\left(1-{\frac{\overline{T}\cdot \log (64k)}{T}}-\frac{4}{T}\right) \left( 1 -8c\sqrt{\frac{ k \log T }{\mu^*T} } \right) \prob\{ E \}                   \\
     & \geq (\mu^*)^{\frac{w}{T}}\left(1-{\frac{\overline{T}\cdot \log (64k)}{T}}-\frac{4}{T}\right) \left( 1 -8c\sqrt{\frac{ k \log T }{\mu^*T} } \right) \left( 1 - \frac{4}{T} \right) \tag{via Lemma \ref{lemma:modifiedgoodeventpr}}\\
     & \geq (\mu^*)^{\frac{w}{T}} \left(1-{\frac{\overline{T}\cdot \log (64k)}{T}}-8c\sqrt{\frac{ k \log T }{\mu^*T} } -\frac{8}{T}\right)   \\                                              
     &= (\mu^*)^{\frac{w}{T}} \left(1-{\frac{968c^2\cdot k\log T\cdot \log (64k)}{\mu^* T}}-8c\sqrt{\frac{ k \log T }{\mu^*T} } -\frac{8}{T}\right) \\
    & \geq (\mu^*)^{\frac{w}{T}}\left(1-1000c\sqrt{\frac{ k \log T }{\mu^*T} }\right).
    \end{align*}
Here, the final inequality follows along the lines of the last step in the derivation of (\ref{ineq:CaseOneW}). The lemma stands proved. 
\end{proof}

\subsection{Improved Guarantee for Nash Regret}
\label{section:improved-theorem}

Algorithm \ref{algo:ncb:modified} not only serves as a subroutine in our anytime algorithm (Algorithm \ref{algo:ncb:anytime} in Section \ref{section:final-anytime}), it also provides an improved (over Theorem \ref{thm:nucb}) Nash regret guarantee for settings in which the horizon of play $T$ is known in advance.  In particular, invoking Algorithm \ref{algo:ncb:modified} with $W = T$ we obtain Theorem \ref{theorem:improvedNashRegret} (stated next).

\begin{restatable}{theorem}{TheoremImprovedNashRegret}
\label{theorem:improvedNashRegret}
For any bandit instance with $k$ arms and given any (moderately large) $T$, there exists an algorithm that achieves Nash regret 
\begin{align*}
\NRg_T = O \left( \sqrt{ \frac{k  \log T}{T} } \right).
\end{align*}
\end{restatable}
\begin{proof}
 The stated Nash regret guarantee follows directly by applying  Lemma \ref{lem:modified_ncb} with $w = T$. Specifically, $\NRg_T  \leq \mu^* - (\mu^*)^{\frac{T}{T}}\left(1-1000c\sqrt{\frac{ k \log T }{\mu^*T} } \right) = 1000c\sqrt{\frac{ \mu^* k \log T }{T} } \leq 1000c\sqrt{\frac{ k \log T }{T} }$.
 This completes the proof of the theorem. 
\end{proof}

\subsection{Anytime Algorithm}
\label{section:final-anytime}

As mentioned previously, our anytime algorithm (Algorithm \ref{algo:ncb:anytime}) builds upon the standard doubling trick. The algorithm starts with a guess for the time horizon, i.e., a {window} of length $W \in \mathbb{Z}_+$. Then, for $W$ rounds it either (i) performs uniform exploration, with probability $\frac{1}{W^2}$, or (ii) invokes Algorithm \ref{algo:ncb:modified} as a subroutine (with the remaining probability $\left( 1- \frac{1}{W^2}\right)$). This execution for $W$ rounds completes one {\em epoch} of Algorithm \ref{algo:ncb:anytime}. In the subsequent epochs, the algorithm doubles the window length and repeats the same procedure till the end of the time horizon, i.e., till a stop signal is received.

\begin{algorithm}[ht!]
    \caption{Anytime Algorithm for Nash Regret}
    \label{algo:ncb:anytime}
    \noindent
    \textbf{Input:} Number of arms $k$

    \begin{algorithmic}[1]
        \STATE Initialize $ W = 1$.
        \WHILE{the MAB process continues}
        \STATE With probability $\frac{1}{W^2}$ set {\rm flag} = \textsc{Uniform} , otherwise, with probability $\left(1 - \frac{1}{W^2}\right)$, set {\rm flag }= \textsc{NCB}
        \IF {{\rm flag} = \textsc{Uniform} }
        \FOR{$t=1$ to $W$}
        \STATE Select $I_t$ uniformly at random from $[k]$. Pull arm $I_t$ and observe reward $X_t$. \label{step:calluniform}
        \ENDFOR
        \ELSIF {{\rm flag }= \textsc{NCB}}
        \STATE  Execute Modified NCB($k$, $W$). \label{step:callsubroutine}
        \ENDIF
        \STATE Update $W \gets 2 \times W$. 
        
        \ENDWHILE
    \end{algorithmic}
\end{algorithm}

Algorithm \ref{algo:ncb:anytime} gives us Theorem \ref{thm:anytime-nucb} (stated next and proved in Section \ref{section:proof-of-anytime}). 

\begin{restatable}{theorem}{AnyTimeTheorem}
\label{thm:anytime-nucb}
There exists an anytime algorithm that, at any (moderately large) round $T$, achieves a Nash regret 
\begin{align*}
\NRg_T = O \left( \sqrt{ \frac{k  \log T}{T} } \log T\right).
\end{align*}
\end{restatable}

We will, throughout, use $h$ to denote an epoch index in Algorithm \ref{algo:ncb:anytime} and the corresponding window length as $W_h$. Note that $h = \log W_h +1$. Also, let $e$ denote the total number of epochs during the $T$ rounds of Algorithm \ref{algo:ncb:anytime}. We have $e \leq  \log _2 T+1$. Furthermore, write $R_h$ to denote the number of rounds before the $h$th epoch begins, i.e., $R_h = \sum_{z=1}^{h-1} W_z$. Note that $R_1 = 0$ and the $h$th epoch starts at round $(R_h + 1)$. In addition, let $h^*$ be the smallest value of $h$ for which $W_h\geq \sqrt{T}$, i.e., $h^*$ is the first epoch in which the window length is at least $\sqrt{T}$. 

We next provide three claims connecting these constructs.  

\begin{claim}\label{lem:window}
$W_h=R_{h}+1$.
\end{claim}
\begin{proof}
In Algorithm \ref{algo:ncb:anytime}, the window size doubles after each epoch, i.e., $W_z=2^{z-1}$ for all epochs $z\geq 1$. Therefore, 
$R_h = \sum_{z=1}^{h-1}W_z  =\sum_{z=1}^{h-1} 2^{z-1} =2^{h-1}-1 =W_h-1$. 
Hence, we have $W_h=R_h+1$.
\end{proof}

The next claim notes that the window length in the last epoch, $e$, is at most the horizon of play $T$. 
\begin{claim} \label{claim:we}
$W_e\leq T$.
\end{claim}
\begin{proof}
The definition of $R_h$ implies that each epoch $h$ starts at round $R_h + 1$. Hence, the last epoch $e$, in particular, starts at round $R_e +1$. Indeed, $R_e + 1 \leq T$ and, via Claim \ref{lem:window}, we get that $W_e \leq T$.  
\end{proof}

Recall that $h^*$ denotes the smallest value of $h$ for which $W_h\geq \sqrt{T}$. 
\begin{claim}\label{claim3:anytime}
$R_{h^*}< 2\sqrt{T}$.
\end{claim}
\begin{proof}
By definition of $h^*$, we have $W_{h^*-1}<\sqrt{T}$. Also, note that $R_{h^*} = R_{h^* - 1} + W_{h^*-1} = 2 W_{h^*-1} -1 < 2 \sqrt{T}$; here, the last equality follows from Claim \ref{lem:window}.
\end{proof}

The following lemma provides a bound on the Nash social welfare accumulated by Algorithm \ref{algo:ncb:anytime} in an epoch $h \geq h^*$. 
\begin{lemma}\label{conditional:NCB:anytime}
In any MAB instance with mean $\mu^* \geq \frac{512\sqrt{k\log T}}{\sqrt T}$, the following inequality holds for each epoch $h \geq h^*$ and all rounds $r \in \{R_{h}+1, R_{h} + 2, \ldots, R_{h+1}\}$:
    \begin{align*}
        \left(\prod_{t= R_h + 1}^{r}\E\left[\mu_{I_t}\right]\right)^\frac{1}{T} & \geq (\mu^*)^\frac{r-R_h}{T} \left(1-1001c\sqrt{\frac{ k \log T }{\mu^*T} }\right).
    \end{align*}    
\end{lemma}
\begin{proof}
Fix any epoch $h \geq h^*$ and let $F_h$ denote the event that Algorithm \ref{algo:ncb:anytime} executes the Modified NCB algorithm (Algorithm \ref{algo:ncb:modified}) in epoch $h$; see Line \ref{step:callsubroutine}. Note that $\prob \{F_h \} = 1- \frac{1}{W_h^2}$. The definition of $h^*$ and Claim \ref{claim:we} give us $\sqrt{T}\leq W_h \leq T$. Hence, we can apply Lemma \ref{lem:modified_ncb} with event $F_h$ to obtain:  
\begingroup
\allowdisplaybreaks
    \begin{align*}
    \left(\prod_{t= R_h+1}^{r}\E\left[\mu_{I_t}\right]\right)^\frac{1}{T} 
    &\geq \left(\prod_{t= R_h+1}^{r}\E\left[\mu_{I_t} | F_h \right] \ \prob \{ F_h \} \right) ^\frac{1}{T}
    \\
    &\geq \left(\prod_{t= R_h+1}^{r}\E\left[\mu_{I_t} | F_h \right] \left( 1- \frac{1}{W_h^2}\right) \right) ^\frac{1}{T}\\
    &\geq \left( 1- \frac{1}{W_h^2}\right)  \left(\prod_{t= R_h+1}^{r}\E\left[\mu_{I_t} | F_h \right]  \right) ^\frac{1}{T}\\
    &\geq (\mu^*)^\frac{r-R_h}{T}\left( 1- \frac{1}{W_h^2}\right) \left(1-1000c\sqrt{\frac{ k \log T }{\mu^*T} }\right)  \tag{Lemma \ref{lem:modified_ncb}}\\
    &\geq (\mu^*)^\frac{r-R_h}{T}\left( 1- \frac{1}{T}\right)\left(1-1000c\sqrt{\frac{ k \log T }{\mu^*T} }\right) \tag{since $W_h \geq \sqrt{T}$}\\
   &\geq (\mu^*)^\frac{r-R_h}{T}\left(1-1000c\sqrt{\frac{ k \log T }{\mu^*T} }-\frac{1}{T}\right)\\ 
   &\geq (\mu^*)^\frac{r-R_h}{T}\left(1-1001c\sqrt{\frac{ k \log T }{\mu^*T} }\right).  	
    \end{align*}
\endgroup
The lemma stands proved 
\end{proof}

\subsubsection{Proof of Theorem \ref{thm:anytime-nucb}}
\label{section:proof-of-anytime}
For establishing the theorem, we focus on MAB instances in which the optimal mean $\mu^* \geq \frac{512\sqrt{k\log T}}{\sqrt T}$; otherwise, the stated guarantees on the Nash regret directly holds. 

Recall that $h^*$ denotes the smallest value of $h$ for which $W_h\geq \sqrt{T}$. We will bound the Nash social welfare accrued by Algorithm \ref{algo:ncb:anytime} by first considering the initial $R_{h^*}$ rounds and then separately analyzing the remaining rounds. 

For the first $R_{h^*}$ rounds, note that for every epoch $g \leq h^*$ we have $W_g < \sqrt{T}$.  Hence, for each such epoch $g$, Algorithm \ref{algo:ncb:anytime} executes uniform sampling with probability $\frac{1}{W_g^2} \geq \frac{1}{T}$; see Line \ref{step:calluniform}. Therefore, for all rounds $t \leq R_{h^*}$, we have 
$\E \left[\mu_{I_t} \right] \geq \frac{\mu^*}{k} \frac{1}{T}$. This bound gives us 
    \begin{align}
        \left(\prod_{t=1}^{R_{h^*}} \E\left[\mu_{I_t}\right]\right)^{\frac{1}{T}} &\geq \left( \frac{\mu^*} {kT}\right)^\frac{R_{h^*}}{T}\nonumber \\
&= (\mu^*)^{\frac{R_{h^*}}{T}} \left( \frac{1}{2}\right)^{\frac{R_{h^*}\log{(kT)}}{T}}\nonumber \\
&\geq (\mu^*)^{\frac{R_{h^*}}{T}}  \left( 1 - \frac{R_{h^*}\log{(kT)}}{T} \right) \tag{via Claim \ref{lem:binomial}}\\
& \geq   (\mu^*)^{\frac{R_{h^*}}{T}}  \left( 1 - \frac{2\log{(kT)}}{\sqrt{T}} \right) \label{ineq:anytimeFirstPart}
    \end{align}
Here, the last inequality follows from Claim \ref{claim3:anytime}. 

For the remaining $T -R_{h^*}$ rounds, we perform an epoch-wise analysis and invoke Lemma \ref{conditional:NCB:anytime}. Specifically,
\begingroup
\allowdisplaybreaks
    \begin{align}
        \left(\prod_{t=R_{h^*}+1}^{T} \E\left[\mu_{I_t}\right]\right)^{\frac{1}{T}}
        &=  \left(\prod_{h=h^*}^{e-1} \ \prod_{t= R_h + 1}^{R_{(h+1)}}\E\left[\mu_{I_t}\right]\right)^\frac{1}{T} \cdot    \left(\prod_{t= R_{e}+1}^{T}\E\left[\mu_{I_t}\right]\right)^\frac{1}{T} \nonumber \\
       & =  \prod_{h=h^*}^{e-1} \left( \prod_{t= R_h+1}^{R_h+W_h}\E\left[\mu_{I_t}\right]\right)^\frac{1}{T} \cdot   \left(\prod_{t= R_{e}+1}^{T}\E\left[\mu_{I_t}\right]\right)^\frac{1}{T} \nonumber \\
        &\geq \prod_{h=h^*}^{e-1} (\mu^*)^{\frac{W_h}{T}} \left(1-1001c\sqrt{\frac{ k \log T }{\mu^*T} }\right)\cdot (\mu^*)^{\frac{T-R_e}{T}} \left(1-1001c\sqrt{\frac{ k \log T }{\mu^*T} }\right) \tag{via Lemma \ref{conditional:NCB:anytime}} \\
        &\geq  (\mu^*)^{1 - \frac{R_{h^*}}{T}} \prod_{j=1}^{e} \left(1-1001c\sqrt{\frac{ k \log T }{\mu^*T} }\right) \nonumber \\ 
        &\geq (\mu^*)^{1 - \frac{R_{h^*}}{T}}  \left(1-1001c\sqrt{\frac{ k \log T }{\mu^*T} }\right) ^ {\log (2T)} \tag{since $e\leq \log T+1$}\\
        &\geq   (\mu^*)^{1 - \frac{R_{h^*}}{T}}  \left(1-1001c\frac{ \sqrt{k \log T} \log (2T) }{\sqrt{\mu^*T}} \right) \label{ineq:anytimeSecondPart}
    \end{align}
\endgroup
The last inequality follows from the fact that $(1-x)(1-y) \geq 1 - x- y$, for all $x,y \geq 0$.

Inequalities (\ref{ineq:anytimeFirstPart}) and (\ref{ineq:anytimeSecondPart}), give us an overall bound on the Nash social welfare of Algorithm \ref{algo:ncb:anytime}: 
    \begin{align*}
        \left(\prod_{t=1}^{T} \E\left[\mu_{I_t}\right] \right)^{\frac{1}{T}} & = \left(\prod_{t=1}^{R_{h^*}} \E\left[\mu_{I_t}\right]\right)^{\frac{1}{T}} \left(\prod_{t=R_{h^*}+1}^{T} \E\left[\mu_{I_t}\right]\right)^{\frac{1}{T}} \\
        &\geq \mu^*  \left( 1 - \frac{2\log{(kT)}}{\sqrt{T}} \right) \left(1-1001c\frac{ \sqrt{k \log T} \log (2T) }{\sqrt{\mu^*T}} \right) \\ 
        &\geq \mu^* \left ( 1 - \frac{2\log{(kT)}}{\sqrt{T}} - 1001c\frac{ \sqrt{k \log T} \log (2T) }{\sqrt{\mu^*T}} \right). \\ 
    \end{align*}
Therefore, the Nash regret of Algorithm \ref{algo:ncb:anytime} satisfies     
\begin{align*}
\NRg_T = \mu^*- \left( \prod_{t=1}^{T} \E\left[\mu_{I_t}\right] \right)^{\frac{1}{T}} \leq \frac{2\mu^*\log{(kT)}}{\sqrt{T}} + 1001c\frac{ \sqrt{\mu^* k \log T} \log (2T) }{\sqrt{T}} \leq 1003c\frac{ \sqrt{k \log T} \log (2T) }{\sqrt{T}}.
\end{align*}
The theorem stands proved. 

%% file: conclusion.tex
\section{Conclusion and Future Work}
This work considers settings in which a bandit algorithm's expected rewards, $\{\E \left[\mu_{I_t}\right] \}_{t=1}^T$, correspond to values distributed among $T$ agents. In this ex ante framework, we apply Nash social welfare (on the expected rewards) to evaluate the algorithm's performance and thereby formulate the notion of Nash regret. Notably, in cumulative regret, the algorithm is assessed by the social welfare it generates. That is, while cumulative regret captures a utilitarian objective, Nash regret provides an axiomatically-supported primitive for achieving both fairness and economic efficiency.      

We establish an instance-independent (and essentially tight) upper bound for Nash regret. Obtaining a Nash regret bound that explicitly depends on the gap parameters, $\Delta_i \coloneqq \mu^* - \mu$, is an interesting direction of future work. It would also be interesting to formulate regret under more general welfare functions. Specifically, one can consider the generalized-mean welfare \cite{moulin2004fair} which---in the current context and for parameter $p \in (-\infty, 1]$---evaluates to $\left( \frac{1}{n} \sum_t \E \left[\mu_{I_t}\right]^p \right)^{1/p}$. Generalized-means encompass various welfare functions, such as social welfare ($p=1$), egalitarian welfare ($p \to -\infty$), and Nash social welfare ($p \to 0$). Hence, these means provide a systematic tradeoff between fairness and economic efficiency. Studying Nash regret in broader settings---such as contextual or linear bandits---is a meaningful research direction as well.

%% file: appendix-prob-G.tex
\section{Proof of Lemma \ref{lemma:good-event}}
\label{appendix:prob-G}
In this section we prove that the good event $G$ holds with probability at least $\left(1 - \frac{4}{T} \right)$. Towards this, we first state two standard concentration inequalities, Lemmas \ref{lemma:hoeffding} and \ref{lemma:hoeff_small_mean}; see, e.g., \cite[Chapter~1.6]{dubhashi2009concentration}.

\begin{lemma}[Hoeffding's Inequality] \label{lemma:hoeffding}
	Let $Y_1,Y_2,\ldots, Y_n$ be independent random variables distributed in $[0,1]$. Consider their average $\widehat{Y} \coloneqq \frac{Y_1+\ldots+Y_n}{n}$ and let $\nu = \E[\widehat{Y}]$ be its expected value. Then, for any $ 0 \leq \delta \leq1 $,
	\begin{align*}
		\prob \left\{ \left|\widehat{Y} - \nu \right| \geq \delta \nu \right\} \leq 2 \ \mathrm{exp} \left(-\frac{\delta^2}{3} \ n \nu\right).
	\end{align*}
\end{lemma}

\begin{lemma}[Hoeffding Extension] \label{lemma:hoeff_small_mean}
	Let $Y_1,Y_2,\ldots, Y_n$ be independent random variables distributed in $[0,1]$. Consider their average $\widehat{Y} \coloneqq \frac{Y_1+\ldots+Y_n}{n}$ and suppose $\E[\widehat{Y}] \leq \nu_H$. Then, for  any $ 0 \leq \delta \leq 1 $,
	\begin{align*}
		\prob \left\{ \widehat{Y} \geq \left(1+ \delta \right) \nu_H \right\} \leq  \ \mathrm{exp} \left(-\frac{\delta^2}{3} n \nu_{H} \right).
	\end{align*}
\end{lemma}

Using these concentration bounds, we establish two corollaries for the empirical means of the arms.

\begin{corollary}\label{cor:hoef}
Consider any arm $i$, with mean $\mu_i> \frac{6 \sqrt{k \log k\log T}}{ \sqrt{T}}$, and sample count $n$ such that $\frac{\widetilde{T}}{2k} \leq n \leq T  $. Let $\widehat{\mu}_i$ be the empirical mean of arm $i$'s rewards, based on $n$ independent draws. Then,
\begin{align*}
\prob \left\{ \left|\mu_i-\widehat{\mu_i}\right|\geq 3\sqrt{\frac{\mu_i\log T}{n}}\right\} & \leq \frac{2}{T^{3}}. 
\end{align*}
\end{corollary}
\begin{proof}
We apply Lemma \ref{lemma:hoeffding} (Hoeffding's inequality), with $Y_s$ as the $s$th independent sample from arm $i$ and for all $1\leq s \leq n$. Also,  we instantiate the lemma with $\delta = 3\sqrt{\frac{\log T}{\mu_i n}}$ and note that $\delta <1$, since $\mu_i> \frac{6 \sqrt{k\log k \log T}}{ \sqrt{T}}$ and $n \geq \frac{\widetilde{T}}{2k} = \frac{8}{k} \sqrt{\frac{k T \log T}{\log k}}$. Specifically, 
	\begin{align*}
		\prob \left\{ \left|\mu_i-\widehat{\mu_i}\right| \geq 3 \sqrt{\frac{\mu_i\log T}{n}}\right\}
		 & = \prob \left\{ \left|\mu_i-\widehat{\mu_i}\right|\geq 3\sqrt{\frac{\log T}{\mu_i n}} \  \mu_i\right \} \\
		 & \leq 2\ \mathrm{exp}\left(-\frac{9 \log T }{3\mu_in} \ n\mu_i \right)          \tag{via Lemma \ref{lemma:hoeffding}}                      \\
		 & =2 \  \mathrm{exp}\left(- 3 \log T \right)                                                      \\
		 & = \frac{2}{T^{3}}.
	\end{align*}
\end{proof}

\begin{corollary}\label{cor:hoef_small_mean}
Consider any arm $j$, with mean $\mu_j \leq \frac{6 \sqrt{k \log k \log T}}{ \sqrt{T}}$, and sample count $n$, such that $\frac{\widetilde{T}}{2k} \leq n \leq T  $. Let $\widehat{\mu}_j$ be the empirical mean of arm $j$'s rewards, based on $n$ independent draws. Then,
	\begin{equation*}
		\prob \left\{ \widehat{\mu}_j \geq \frac{9 \sqrt{k\log k \log T}}{\sqrt{T}} \right\} \leq \frac{1}{T^{3}}. 
	\end{equation*}
\end{corollary}
\begin{proof}
We invoke Lemma \ref{lemma:hoeff_small_mean}, with $\delta = 1/2$ and {$\nu_H = \frac{6 \sqrt{k \log k \log T}}{ \sqrt{T}}$}, to obtain 
	\begin{align*}
		\prob \left\{ \widehat{\mu}_j \geq \frac{9 \sqrt{k \log k\log T}}{\sqrt{T}} \right\} 
		 & = \prob \left\{ \widehat{\mu}_j \geq (1+\frac{1}{2})\nu_H \right\}  \\
		 & \leq \mathrm{exp}\left(-\frac{1}{12} \ n \ \frac{6 \sqrt{k \log k\log T}}{ \sqrt{T}} \right)                                 \\
		 & \leq \mathrm{exp}\left(-\frac{1}{12} \ \frac{8}{k} \sqrt{\frac{k T \log T}{\log k}} \ \frac{6 \sqrt{k \log k\log T}}{ \sqrt{T}} \right)   \tag{since $n \geq \frac{8}{k} \sqrt{ \frac{k T \log T}{\log k} } $}                                                   \\
		 &  = \mathrm{exp}\left( - 4 \log T \right) \\
		& \leq \frac{1}{T^{3}}.
	\end{align*}
\end{proof}

Along with Corollary \ref{cor:hoef} and \ref{cor:hoef_small_mean}, we will invoke the Chernoff Bound (stated next).


\begin{lemma}[Chernoff Bound]\label{lem:chernoff}
	Let $Z_1, \ldots, Z_n$ be independent Bernoulli random variables. Consider the sum $S = \sum_{r=1}^n Z_r$ and let $\nu = \E[S]$ be its expected value. Then, for any $\varepsilon \in [0,1]$, we have 
	\begin{align*}
		\prob \left\{ S \leq (1-\varepsilon) \nu \right\} & \leq \mathrm{exp} \left( -\frac{\nu \varepsilon^2}{2} \right), \text{  and} \\ 
		\prob \left\{ S \geq (1+\varepsilon) \nu \right\} & \leq \mathrm{exp} \left( -\frac{\nu \varepsilon^2}{3} \right).
	\end{align*}
\end{lemma}

We now prove Lemma \ref{lemma:good-event} by  bounding the probabilities of the three sub-events $G_1$, $G_2$, and $G_3$, respectively. Recall that $G = G_1 \cap G_2 \cap G_3$.  

For $G_1^c$ (i.e., the complement of $G_1$), we will invoke Lemma \ref{lem:chernoff} for every arm and then apply the union bound. In particular, fix any arm $i $ and write random variable $Z_r$ to indicate whether arm $i$ is selected in round $r$ of Phase {\rm I}, or not. That is, $Z_r = 1$ if arm $i$ is picked in round $r$ and, otherwise, $Z_r =0$. Note that $n_i$, the number of times arm $i$ is sampled in Phase {\rm I}, satisfies $n_i = \sum_{r=1}^{\widetilde{T}} Z_r$. Now, using the fact that the algorithm selects an arm uniformly at random in every round of Phase {\rm I} and setting $\varepsilon = 1/2$ along with $\nu = \frac{\widetilde{T}}{k}=16\sqrt{\frac{T\log T}{k\log k}}$, Lemma \ref{lem:chernoff} gives us 
\begin{align}
\prob \left\{ n_i <  \frac{\widetilde{T}}{2k} \right\} \leq \mathrm{exp} \left(- \frac{16\sqrt{T \log T }}{8\sqrt{k\log k}}\right) \leq \frac{1}{T^2} \label{ineq:periG}
\end{align} 
Here, the last inequality follows from the theorem assumption that $T$ is sufficiently large; specifically, $T \geq  \left(k \log k\right)^2$ suffices. Inequality (\ref{ineq:periG}) and the union bound give us 
\begin{align}
\prob\left\{G_1^c\right\} & \leq \frac{1}{T^2} k\leq \frac{1}{T} \label{ineq:probG1c}
\end{align}

Next, we address $G_2^c$. Note that the arms and counts considered in $G_2$, respectively, satisfy the assumption in Corollary \ref{cor:hoef}. Hence, the corollary ensures that, for each arm $i$, with mean  $\mu_i > \frac{6 \sqrt{k \log k \log T}}{\sqrt T}$ and each count $s \geq \frac{\widetilde{T}}{2k}$, we have $\prob \left\{ \left|\mu_{i}-\widehat{\mu}_{i,s} \right| \geq 3\sqrt{\frac{\mu_i\log T}{s}}\right\} \leq \frac{2}{T^{3}}$. Therefore, applying the union bound we get  
\begin{align}
\prob \{ G_2^c \}  & \leq \frac{2}{T^{3}} \ kT \leq \frac{2}{T} \label{ineq:probG2c}
\end{align} 

In addition, Corollary \ref{cor:hoef_small_mean} provides a probability bound for $G_3^c$. The arms and counts considered in $G_3$ satisfy the requirements of Corollary \ref{cor:hoef_small_mean}. Therefore, for any arm $j$, with mean  $ \mu_j \leq \frac{6\sqrt{k \log k\log T}}{\sqrt{T}} $ and count $s \geq \frac{\widetilde{T}}{2k}$, the probability $\prob \left\{ \widehat{\mu}_j \geq \frac{9 \sqrt{k \log k \log T}}{\sqrt{T}} \right\} \leq \frac{1}{T^{3}}$. Again, an application of union bound gives us

\begin{align}
\prob \{ G_3^c \}  & \leq \frac{1}{T^{3}} \ kT \leq \frac{1}{T} \label{ineq:probG3c}
\end{align} 

Inequalities (\ref{ineq:probG1c}), (\ref{ineq:probG2c}), and (\ref{ineq:probG3c}) lead to desired bound 
\begin{align*}
	\prob \left\{ G\right\} & = 1 - \prob \left\{ G^c \right\}  \geq 1 -  \prob \left\{ G_1^c\right\} -  \prob \left\{ G_2^c\right\}-  \prob \left\{ G_3^c\right\}  \geq 1 - \frac{4}{T}.
\end{align*}

%% file: appendix-numeric.tex
\section{Proof of Claim \ref{lem:binomial}}
\label{appendix:numeric-ineq}
This section restates and proves Claim \ref{lem:binomial}.
\LemmaNumeric*

\begin{proof}
The binomial theorem gives us 
	\begin{align}
		(1-x)^a & = 1-ax +\frac{a(a-1)}{2!} x^2 - \frac{a(a-1)(a-2)}{3!} x^3 + \ldots                                       \nonumber       \\
		        & = 1-ax -ax \left(\frac{(1-a)}{2!} x + \frac{(1-a)(2-a)}{3!} x^2 + \frac{(1-a)(2-a)(3-a)}{4!} x^3 + \ldots \right)  \label{eq:bino}
	\end{align}
	
We can bound the multiplied term as follows 
	\begin{align*}
		\frac{(1-a)}{2!} x + \frac{(1-a)(2-a)}{3!} x^2 + & \frac{(1-a)(2-a)(3-a)}{4!} x^3 + \ldots                                           \\
		 & \leq \frac{1}{2!} x + \frac{1\cdot2}{3!} x^2 + \frac{1\cdot2\cdot3}{4!} x^3 + \ldots \tag{since $a \in (0,1)$}\\
		 & = \frac{1}{2} x + \frac{1}{3} x^2 + \frac{1}{4} x^3 \dots                      \\
		 & \leq x + x^2 + x^3 \dots                                                          \\
		 & = \frac{x}{1-x}	\tag{since $x < 1$}
	\end{align*}
Hence, equation (\ref{eq:bino}) reduces to 
\begin{align*}
(1-x)^a \geq 1-ax -ax \frac{x}{1-x}.
\end{align*}
Furthermore, since $x \leq \frac{1}{2}$, the ratio $\frac{x}{1-x} \leq 1$. Therefore, we obtain $(1-x)^a > 1 -2ax$. This completes the proof of the claim. 
\end{proof}

%% file: appendix-key-lemmas.tex
\section{Missing Proofs from Section \ref{section:regret-analysis}}
\label{appendix:key-lemmas}

First, we restate and prove Lemma \ref{lem:emp}.

\LemmaNCBopt*
\begin{proof}
Fix any round $t > \widetilde{T}$ and write $n_{i^*}$ to denote the number of times the optimal arm $i^*$ has been pulled before that round. Also, let $\widehat{\mu}^*$ denote the empirical mean of arm $i^*$ at round $t$. Hence, by definition, at this round the Nash confidence bound ${\NCB_{i^*,t}} \coloneqq  \widehat{\mu}^* + 4 \sqrt{\frac{\widehat{\mu}^* \log T}{n_{i^*}}}$. Note that under the good event $G$ (in particular, under $G_1$) we have $n_{i^*} \geq \frac{\widetilde{T}}{2k}=\frac{8\sqrt{T \log T}}{\sqrt{k\log k}}$. This inequality and the assumption $\mu^* \geq \frac{32\sqrt{k\log k\log T}}{\sqrt T}$ imply \begin{align}
\mu^* \ n_{i^*} \geq \frac{32\sqrt{k\log k\log T}}{\sqrt T} \frac{8\sqrt{T \log T}}{\sqrt{k\log k}} = 256 \log T \label{ineq:munopt}
\end{align} 

In addition, the event $G$ (specifically, $G_2$) gives us 
\begin{align}
\widehat{\mu}^* & \geq \mu^* - 3 \sqrt{\frac{\mu^* \log T}{n_{i^*}}} \nonumber \\ 
& = \mu^* - 3 \mu^* \sqrt{\frac{\log T}{\mu^* n_{i^*}}} \nonumber \\ 
& \geq \mu^* - 3 \mu^* \sqrt{\frac{1}{256}} \tag{via inequality (\ref{ineq:munopt})}\\
& = \frac{13}{16} \mu^* \label{ineq:muhatmustar}
\end{align}
Therefore, 
\begin{align*}
{\NCB_{i^*,t}} & = \widehat{\mu}^* + 4 \sqrt{\frac{\widehat{\mu}^* \log T}{n_{i^*}}} \\ 
& \geq \mu^* - 3 \sqrt{\frac{\mu^* \log T}{n_{i^*}}} + 4 \sqrt{\frac{\widehat{\mu}^* \log T}{n_{i^*}}} \tag{via event $G_2$}\\\
&\geq \mu^* - 3 \sqrt{\frac{\mu^* \log T}{n_{i^*}}} + 4 \sqrt{\frac{13 \mu^* \log T}{16n_{i^*}}} \tag{via inequality (\ref{ineq:muhatmustar})} \\ 
& \geq \mu^* + 0.6 \sqrt{\frac{\mu^* \log T}{n_{i^*}}}.
\end{align*}
The lemma stands proved. 
\end{proof}

We restate and prove Lemma \ref{lem:bad_arms} below. 

\LemmaBadArms*
\begin{proof}
Fix any arm $j$ with mean $\mu_j \leq \frac{6\sqrt{k \log k\log T}}{\sqrt{T}}$. Let $r_j$ denote the number of times arm $j$ is pulled in Phase {\rm I}. 
Under event $G$ (in particular, $G_1$) we have $r_j \geq \frac{\widetilde{T}}{2k}$. In such a case, the good event (in particular, $G_3$) additionally ensures that (throughout Phase {\rm II}) the empirical mean of arm $j$ satisfies: $\widehat{\mu}_j < \frac{9\sqrt{k\log k \log T}}{\sqrt{T}}$. 

Furthermore, under the good event, 
$\NCB_{i^*,t} \geq \mu^*$ for all rounds $t$ in Phase {\rm II} (Lemma \ref{lem:emp}). For any round $t$ in Phase {\rm II} (i.e., for any $t \geq {\widetilde{T}} = 16\sqrt{\frac{k T \log T}{\log k}}$), write $\NCB_{j,t}$ to denote the Nash confidence bound of arm $j$ at round $t$. Below we show that the $\NCB_{j, t} $ is strictly less than $\NCB_{i^*,t}$ and, hence, arm $j$ is not even pulled once in all of Phase {\rm II}.
	\begin{align*}
		\NCB_{j,t} & = \widehat{\mu}_j + 4\sqrt{\frac{\widehat{\mu}_j \log T}{r_j}} \\ 
		&\leq \frac{9\sqrt{ k\log k \log T}}{\sqrt{T}} + 4\sqrt{\frac{9\sqrt{ k\log k \log T}}{\sqrt{T}} \cdot \frac{\log T}{r_j}} \tag{via event $G_3$}	 \\ 
		& \leq \frac{9\sqrt{ k \log k\log T}}{\sqrt{T}} + 4\sqrt{\frac{9 \sqrt{ k\log k \log T}}{\sqrt{T}} \cdot \frac{ k \log T \sqrt{\log k}}{8\sqrt{Tk \log T}}} \tag{since $r_j \geq \frac{\widetilde{T}}{2k}$ under $G_1$}\\ 
		&< \frac{32 \sqrt{k\log k\log T}}{\sqrt T} \quad \leq \mu^* \quad \leq \NCB_{i^*, t} \tag{via Lemma \ref{lem:emp}}
	\end{align*}
This completes the proof of the lemma. 
\end{proof}

Finally, we prove Lemma \ref{lem:suboptimal_arms}.

\LemmaHighMean*
\begin{proof}
Any arm $j$ with mean $\mu_j \leq \frac{6\sqrt{k \log k \log T}}{\sqrt{T}}$ is never pulled in Phase {\rm II} (Lemma \ref{lem:bad_arms}). Hence, we focus on arms $i$ with $\mu_i > \frac{6\sqrt{k \log k \log T}}{\sqrt{T}}$. Note that when arm $i$ was pulled the $T_i$th time during Phase {\rm II}, it must have had the maximum Nash confidence bound value; in particular, at that round $\NCB_i \geq \NCB_{i^*} \geq \mu^*$. Here, the last inequality follows from Lemma  \ref{lem:emp}. Therefore, with $\widehat{\mu}_i$ denoting the empirical mean of arm $i$ at this point, we have 
\begin{align}
\widehat{\mu}_i + 4\sqrt{\frac{\widehat{\mu}_i \log T}{T_i -1 }} \geq \mu^* \label{ineq:hatmu} 
\end{align}
Complementing inequality (\ref{ineq:hatmu}), we will now upper bound the empirical mean $\widehat{\mu}_i$ in terms of $\mu^*$.  Note that, since arm $i$ is pulled at least once in Phase {\rm II} and event $G_1$ holds, we have $T_i > \frac{ \widetilde{T}}{2k}= \frac{ 8\sqrt{T \log T}}{\sqrt{k\log k}}$. Using this fact and event $G_2$ we obtain  
\begin{align}
\widehat{\mu_i} & \leq  \mu_i + 3\sqrt{\frac{\mu_i \log T}{T_i -1}}                \tag{since $G_2$ holds} \nonumber \\
& \leq \mu^* + 3\sqrt{\frac{\mu^* \log T}{ \frac{8\sqrt{T \log T}}{\sqrt{k\log k}}}} \tag{since $T_i>\frac{ 8\sqrt{T \log T}}{\sqrt{k\log k}}$ and $\mu_i \leq \mu^*$}       \nonumber \\
& = \mu^* + 3 \sqrt{\frac{ \mu^* \sqrt{k\log k \log T}}{8\sqrt{T}}} \nonumber \\
& \leq \mu^* + 3\sqrt{\frac{ \mu^* \mu^* }{ 256}}   \tag{since $\mu^* \geq \frac{32  \sqrt{k\log k\log T}}{\sqrt{T}}$} \nonumber \\
& = \frac{19}{16} \mu^* \label{ineq:hatmuo}
\end{align}

Inequalities (\ref{ineq:hatmu}) and (\ref{ineq:hatmuo}) give us 
\begin{align*}
& \mu^* \leq \widehat{\mu}_i + 4\sqrt{\frac{19\mu^* \log T}{16(T_i -1)}}                                                          \\
& \leq \mu_i + 3\sqrt{\frac{\mu_i \log T}{T_i -1 }} + 4\sqrt{\frac{19\mu^* \log T}{16(T_i -1)}}   \tag{via event $G_2$} \\
& \leq \mu_i + 3\sqrt{\frac{\mu^* \log T}{T_i  -1}} + 4\sqrt{\frac{19\mu^* \log T}{16(T_i -1 )}}                    \tag{since $\mu_i \leq \mu^*$}         \\
& \leq \mu_i +  8\sqrt{\frac{\mu^* \log T}{T_i -1 }}.
\end{align*}
This completes the proof of the lemma. 
\end{proof}

%% file: appendix-improved-anytime.tex
\section{Missing Proofs from Section \ref{section:modified-NCB}}

\subsection{Proof of Lemma \ref{lemma:modifiedgoodeventpr}}
\label{appendix:proof-of-E}
This section shows that the good event $E$ occurs with probability at least $\left( 1- \frac{4}{T} \right)$. Toward this, we will upper bound the probabilities of the complements of the events $E_1$, $E_2$, and $E_3$, respectively, and apply union bound to establish the lemma. 

For $E_{1}^c$, we invoke the Chernoff bound (Lemma \ref{lem:chernoff}) with random variable $Z_{i,t}$ indicating whether the arm $i$ is selected in round $t$ of uniform sampling, or not. That is, $Z_{i,t} = 1$ if the arm $i$ is picked in round $t$ and, otherwise, $Z_{i,t} =0$. Using the fact that the algorithm selects an arm uniformly at random in every round of Phase \textcal{1} and setting $\varepsilon = 1/2$ along with $\nu =\frac{r}{k}\geq 128 \Sample$, Lemma \ref{lem:chernoff} along with union bound gives us\footnote{Recall that $\Sample = \frac{c^2 \log T}{ \mu^*}$ and $\mu^* \leq 1$.}
\begin{align}
    \mathbb{P} \left\{ E_{1}^c\right\} & \leq 2\cdot \mathrm{exp} \left(- \frac{128\cdot c^2\log T}{12\mu^*}\right) \cdot kT \leq \frac{1}{T} \label{ineq:probG1c:anytime}
\end{align}

Next, we address $E_{2}^c$. Note that for each arm $i$, with mean  $\mu_i > \frac{\mu^*}{64}$, and for each count $s \geq 64 \Sample$, we have $c\sqrt{\frac{\log T}{\mu_is}} < 1$. Hence, Lemma \ref{lemma:hoeffding} gives us 
\begin{align*}
\mathbb{P} \left\{ \left|\mu_{i}-\widehat{\mu}_{i,s} \right|  \geq c\sqrt{\frac{\mu_i\log T}{s}}\right\} & = \prob\left\{ \left|\mu_{i}-\widehat{\mu}_{i,s} \right| \geq c\sqrt{\frac{ \log T}{\mu_i \  s}} \mu_i  \right\}  
\leq 2\ \mathrm{exp}\left(- \frac{c^2 \log T}{3 \mu_i \ s } s \mu_i \right) 
= \frac{2}{T^{3}}. 
\end{align*}
Therefore, via the union bound we obtain  
\begin{align}
    \mathbb{P} \{ E_{2}^c \} & \leq \frac{2}{T^{3}} \ kT \leq \frac{2}{T} \label{ineq:probG2c:anytime}
\end{align}

Finally, we address $E_{3}^c$. Consider any arm $j$, with mean  $ \mu_j \leq  \frac{\mu^*}{64}$ and any count $s \geq 64 \Sample$. Lemma \ref{lemma:hoeff_small_mean} (applied with $\nu_H=\frac{\mu^*}{64}$ and $\delta=1$) leads to 
\begin{align*}
\mathbb{P} \left\{ \widehat{\mu}_{j,s} \geq \frac{\mu^*}{32} \right\} \leq  \mathrm{exp} \left(-\frac{1}{3} \ \frac{\mu^*}{64} \ s \right)  = \mathrm{exp} \left(-\frac{c^2 \log T}{3} \right)  = \frac{1}{T^{3}}.
\end{align*}
 Again, an application of union bound gives us

\begin{align}
    \mathbb{P} \{ E_{3}^c \} & \leq \frac{1}{T^{3}} \ kT \leq \frac{1}{T} \label{ineq:probG3c:anytime}
\end{align}
Inequalities (\ref{ineq:probG1c:anytime}), (\ref{ineq:probG2c:anytime}), and (\ref{ineq:probG3c:anytime}) establish the lemma:
\begin{align*}
    \mathbb{P} \left\{ E\right\} & = 1 - \mathbb{P} \left\{ E^c \right\}  \geq 1 -  \mathbb{P} \left\{ E_{1}^c\right\} -  \mathbb{P} \left\{ E_{2}^c\right\}-  \mathbb{P} \left\{ E_{3}^c\right\}  \geq 1 - \frac{4}{T}.
\end{align*}

\subsection{Proof of Supporting Lemmas}
\label{appendix:modifiedncb-supporting-lemmas}

Here, we restate and prove Lemma \ref{tau:lower}. 

\LemmaTauLower*
\begin{proof}
Write $N \coloneqq 192 \Sample$. Note that, for any arm $i$, the product $n \ \widehat{\mu}_{i,n} $ is equal to the sum of the rewards observed for arm $i$ in the first $n$ samples. Therefore, for all $n \leq N$, we have 
\begin{align}
n \  \widehat{\mu}_{i, n} \leq N \ \widehat{\mu}_{i,N} \label{ineq:sumemp}
\end{align}

Using inequality (\ref{ineq:sumemp}), we first show that the lemma holds for arms $j$ whose mean $\mu_j \leq \frac{\mu^*}{64}$. Note that for any such arm $j$, event $E_3$ gives us $\widehat{\mu}_{j,N} \leq \frac{\mu^*}{32}$. Therefore, 
\begin{align*}
n \ \widehat{\mu}_{j, n} \underset{\text{via (\ref{ineq:sumemp})}}{\leq} N \  \widehat{\mu}_{j,N} \leq 192 \Sample \frac{\mu^*}{32} = 6 c^2 \log T.
\end{align*}
Next, we complete the proof by proving the lemma for arms $i$ whose mean $\mu_i \geq \frac{\mu^*}{64}$. For any such arm $i$, we have
\begin{align*}
\widehat{\mu}_{i,N} & \leq \mu_i + c \sqrt{\frac{\mu_{i} \log T}{N}} \tag{via event $E_2$}\\
& \leq \mu^*+c \sqrt{\frac{\mu^* \log T}{N}}   \tag{since $\mu_i\leq \mu^*$}\\
& =\mu^*+\frac{\mu^*}{\sqrt{192}}     \tag{since $N=192 \Sample = \frac{192 c^2 \log T}{\mu^*}$}\\
& < \frac{210}{192} \mu^*
\end{align*}
Hence, even for arms with high enough means we have $N \ \widehat{\mu}_{i,N}  < \frac{210}{192} \mu^* \ 192 \Sample=210c^2\log T$.
\end{proof}

Lemma \ref{tau:upper} is established next. 

\LemmaTauUpper*
\begin{proof}
Write $M \coloneqq 484 \Sample$ and note that, for all $n \geq M$, we have $n \ \widehat{\mu}_{i,n} \geq M \ \widehat{\mu}_{i,M}$. Furthermore, 
\begin{align*}
\widehat{\mu}_{i^*,M} & \geq \mu^*-c \sqrt{\frac{\mu^* \log T}{M}} \tag{via event $E_2$} \\
& =\mu^*-\frac{\mu^*}{\sqrt{484}}      \tag{since $M=484 \Sample = \frac{484 c^2 \log T}{\mu^*}$}\\
& = \frac{21}{22} \mu^*
\end{align*}
Hence, for any $n \geq M = 484  \Sample$, the total observed reward $n \ \widehat{\mu}_{i,n} \geq \frac{21}{22} \mu^* \ 484 \Sample=462c^2\log T$.
\end{proof}

Next, we restate and prove Lemma \ref{tau:anytime}

\LemmaTauAnytime*
\begin{proof}
Write $t_1 \coloneqq  128 k  \Sample$ and note that event $E$ (specifically, $E_1$) ensures that {\it at} $t_1$ rounds of uniform sampling, no arm has been sampled more than $192 \Sample$ times. This, in fact, implies that no arm gets sampled more than $192 \Sample$ times throughout the first $t_1$ rounds of uniform exploration. Hence, Lemma \ref{tau:lower} implies that, till the round $t_1$, for every arm $i$ the sum of observed rewards $n_i\ \widehat{\mu_i}$ is less than $210c^2\log T \leq 420 c^2 \log W$. Therefore, $\tau \geq 128 k \Sample$.

In addition, let $t_2 \coloneqq 968 k \Sample$. Under event $E$, we have that each arm $i$ is sampled at least $484 \Sample$ times by the $t_2$th round of uniform sampling. Therefore, Lemma \ref{tau:upper} implies that, by round $t_2$ and for the optimal arm $i^*$, the sum of rewards $n_{i^*}\widehat{\mu}_{i^*}$  is at least $462c^2\log T > 420 c^2 \log W$. Hence, $\tau \leq 968 k \Sample$.
This completes the proof of the lemma. 
\end{proof}

%% file: appendix-nash-defns.tex
\section{Other Formulations of Nash Regret}
\label{appendix:defns-nash-regret}

This section compares Nash regret $\NRg_T$ with two variants, $\NRg_T^{(0)}$ and $\NRg_T^{(1)}$, defined below. For completeness, we also detail the relevant aspects of the canonical bandit model (see, e.g., \cite[Chapter 4]{lattimore2020bandit}), which is utilized in analysis of MAB algorithms.  

For a bandit instance with $k$ arms and horizon of play $T$, the canonical model works with a $k \times T$ reward table $(Y_{i,s})_{i \in [k], s \in [T]}$ that populates $T$ independent samples for each arm $i \in [k]$. That is, $Y_{i, s}$ is $s$th independent draw from the $i$th distribution. As before, for a given bandit algorithm, the random variable $I_t \in [k]$ denotes the arm pulled in round $t \in \{1, \ldots, T\}$. Recall that $\mu_{I_t}$ denotes the associated mean. Furthermore, in each round $t$, the reward observed, $X_t$, is obtained from the reward table: $X_t = Y_{I_t, t}$. 

Therefore, the sample space $\Omega \coloneqq \{1,2,\dots,k\}^T \times [0,1] ^{k \times T}$. Each element of the sample space is a tuple $\omega = ((I_1, I_2, \dots, I_T), (Y_{i,s})_{i \in [k], s \in [T]})$, denoting the list of arms pulled at each time $t$ and the rewards table. Precisely, let $\mathcal{F}_1 = \mathfrak{P}(\{1,2,\dots,k\}^T)$ be the power set of $\{1,2,\dots,k\}^T$, $\mathcal{F}_2 = \mathcal{B}(\mathbb{R}^{k\times T})$ be the Borel $\sigma$-algebra on $\mathbb{R}^{k \times T}$, and $\mathcal{F} = \sigma (\mathcal{F}_1 \times \mathcal{F}_2)$ be the product $\sigma$-algebra. The probability measure on $(\Omega , \mathcal{F})$ is induced by the bandit instance and the algorithm. The regret definitions and analysis are based on this probability measure. 

Recall that Nash regret is defined as $\NRg_T \coloneqq \mu^* - \left( \prod_{t=1}^T  \E _{I_t} [\mu_{I_t}]  \right)^\frac{1}{T}$. Towards a variant, one can first consider the difference between the optimal mean,   $\mu^*$, and geometric mean of the realized rewards:
\begin{align*}
\NRg_T^{(0)} \coloneqq \mu^* - \E_{X_1,\ldots X_T}\ \left[ \left( \prod_{t=1}^{T} X_t \right)^\frac{1}{T} \right].
\end{align*}

Indeed, $\NRg^{(0)}_T$ is an unreasonable metric: Even if the algorithm pulls the optimal arm in every round, a single draw $X_t$ can be equal to zero with high probability (in a general MAB instance). In such cases, $\NRg^{(0)}_T$ would be essentially as high as $\mu^*$.  

A second variant is obtained as follows:
\begin{align*}
\NRg_T^{(1)} \coloneqq \mu^* - \E_{I_1,\ldots I_T}\ \left[ \left( \prod_{t=1}^{T} \mu_{I_t} \right)^\frac{1}{T} \right].
\end{align*}
 
While $\NRg_T^{(1)}$ upper bounds Nash regret $\NRg_T$ (see Theorem \ref{appendix:nash-regrets} below), it does not conform to a per-agent ex ante assessment. We also note the regret guarantee we obtain for Phase {\rm II} (of Algorithm \ref{algo:ncb}) in fact holds for $\NRg_T^{(1)}$; see inequality (\ref{ineq:interim}) and the following analysis.

\begin{theorem} \label{appendix:nash-regrets}
For any MAB instance and bandit algorithm we have $\NRg_T^{(0)} \geq \NRg_T^{(1)} \geq \NRg_T$.
\end{theorem}
\begin{proof}
Since the geometric mean is a concave function, the multivariate form of Jensen's inequality gives us $\E \left[ \left( \prod_{t=1}^{T} \mu_{I_t} \right)^\frac{1}{T} \right] \leq \left( \prod_{t=1}^{T} \E\left[\mu_{I_t} \right]\right)^\frac{1}{T}$. Therefore, $\NRg_T^{(1)} \geq \NRg_T$.

Next, we compare $\NRg_T^{(0)}$ and $\NRg_T^{(1)}$. 
We have
\begin{align*}
\E \left[ \left( \prod_{t=1}^{T} X_t \right)^\frac{1}{T} \right] & = \E_{\substack{I_1, \ldots, I_T \\ X_1, \ldots, X_T}} \left[ \left( \prod_{t=1}^{T} X_t \right)^\frac{1}{T} \right] \\
& = \E_{I_1, \ldots, I_T } \left[ \E_{X_1, \ldots, X_T } \left[ \left( \prod_{t=1}^{T} X_t \right)^\frac{1}{T} \;\middle|\; (I_t)_t \right] \ \right] \\
& \leq \E_{I_1, \ldots, I_T } \left[ \left( \prod_{t=1}^T   \E_{X_t} \left[ X_t \;\middle|\; (I_t)_t \right] \right)^\frac{1}{T} \right] \tag{Multivariate Jensen's inequality} \\
& = \E_{I_1, \ldots, I_T } \left[  \left( \prod_{t=1}^T \mu_{I_t} \right)^\frac{1}{T} \right].
\end{align*}  

This last inequality gives us $\NRg_T^{(0)} \geq \NRg_T^{(1)}$. The theorem stands proved. 
\end{proof}

%% file: ucb-bad.tex
\section{Counterexample for the UCB algorithm}
\label{appendix:ucb}

This section shows that, in general, the Nash regret of UCB does not decrease as a function of $T$. In particular, there exist MAB instances in which the UCB algorithm could incur Nash regret close to $1$. 
Recall that, in each round $t$, the UCB algorithm pulls an arm 
\begin{align*}
A_t = \argmax_{i \in [k]} \ \left( \widehat{\mu_i}+\sqrt{\frac{2\log T}{n_{i,t}}} \right),
\end{align*}
here $\widehat{\mu_i}$ denotes the empirical mean of arm $i$'s reward at round $t$ and $n_{i,t}$ denotes the number of times arm $i$ has been pulled before the $t$th round.
Also, write $\UCB_{i,t}$ to denote the upper confidence bound associated with arm $i$, i.e.,  
\begin{align*}
\UCB_{i,t} \coloneqq \widehat{\mu_i}+\sqrt{\frac{2\log T}{n_{i,t}}}.
\end{align*}
UCB algorithm follows an arbitrary, but consistent, tie breaking rule. 


We will next detail an MAB instance that illustrates the high Nash regret of UCB. Consider an instance with two arms, $\arm_1$ and $\arm_2$, and time horizon $T$, such that $T>25 \log T$. Also, let the 
means of the two arms be $\mu_1=(2e)^{-T}$ and $\mu_2=1$, respectively. The rewards of both the arms follow a Bernoulli distribution.\footnote{Hence, $\arm_2$ is a point mass.} 

Write random variable $X_{i,t}$ to denote the reward observed for arm $i$ in the $t$th round. For any given sequence of $X_{i,t}$-s  (as in the canonical bandit model), the order of arm pulls in UCB is fixed. That is, for a given sequence of $X_{i,t}$-s, one can deterministically ascertain the arm that will be pulled in a particular round $r$--this can be done by comparing the values $\UCB_{i,r}$-s and applying the tie-breaking rule. 

Furthermore, write $Z$ to denote the event wherein, for $\arm_1$, the first $T$ pulls yield a reward of $0$. We have $\prob\{ Z \} = \left(1-(2e)^{-T}\right)^{T}\geq 1- (2e)^{-T}T$. 

We will next prove that, under event $Z$, there exist at least $\log T$ rounds in which $\arm_1$ is pulled. Assume, towards a contradiction, that $\arm_1$ is pulled less than $\log T$ times. Since, for $\arm_2$, we have $X_{2, t}=1$ for all rounds $t$, under event $Z$ the sequence of arm pulls in UCB is fixed. Now, consider the round $s$ in which $\arm_2$ is pulled the $\left(\frac{T}{2}+1\right)$th time. Then, it must be the case that $\UCB_{1,s} \leq \UCB_{2, s}$. However, $UCB_{1,s}\geq \sqrt{2}$ (given that $\arm_1$ has been pulled less than $\log T$ many times) and $\UCB_{2,s}=1+\sqrt{\frac{4\log T}{T}}$. This leads to a contradiction and, hence, shows that $\arm_1$ is pulled at least $\log T$ times. Moreover, the rounds in which $\arm_1$ is pulled are fixed under event $Z$.  

Write $\mathcal{R}$ to denote the specific rounds in which $\arm_1$ is pulled under the event $Z$.  Note that $|\mathcal{R}| \geq \log T$. For any $r \in \mathcal{R}$ we have 
\begin{align}
\E [\mu_{I_r}]&=\mathbb{E}[\mu_{I_r}|Z] \ \prob \{Z\} \ + \ \E \left[\mu_{I_r}|Z^c \right] \ \prob\{ Z^c \} \nonumber \\
& \leq \mathbb{E}[\mu_{I_r}|Z] \cdot 1 \ + \E \left[\mu_{I_r}|Z^c \right] \  (2e)^{-T} T \tag{since $\mathbb{P}\{Z^c\} \leq (2e)^{-T} T$}\\
&\leq  (2e)^{-T} \ + \E \left[\mu_{I_r}|Z^c \right] \  (2e)^{-T} T \tag{since $I_r = \arm_1$ under $Z$} \\ 
& \leq (2e)^{-T} \ +   (2e)^{-T}  T  \tag{since $\mathbb{E}[\mu_{I_r} | Z^c]\leq 1$}\\
&\leq e^{-T} \label{ineq:ucbr}
\end{align}
Here, the last inequality follows from the fact that $2^T> T+1$. Also, observe that inequality (\ref{ineq:ucbr}) is a bound on the expected value at round $r$; it holds irrespective of whether $Z$ holds or not. 

Therefore, the Nash social welfare of UCB satisfies 
\begin{align*}
\left(\prod_{t=1}^T\mathbb{E}[\mu_{I_t}]\right)^\frac{1}{T}&\leq \left(\prod_{r \in \mathcal{R}}\mathbb{E}[\mu_{I_r}]\right)^\frac{1}{T} \tag{since $\mathbb{E}[\mu_{I_t}]\leq 1$ for all $t$}\\
&\leq \left(e^{-T}\right)^{\frac{|\mathcal{R}|}{T}} \tag{via inequality (\ref{ineq:ucbr})}\\
&\leq \left(e^{-T}\right)^{\frac{\log T}{T}} \tag{since $|\mathcal{R}|\geq \log T$}\\
&= e^{-\log T}\\
&=\frac{1}{T}
\end{align*}

Hence, the Nash regret of UCB is at least $\left( 1-\frac{1}{T} \right)$.





%% file: nash.bbl
\newcommand{\etalchar}[1]{$^{#1}$}
\begin{thebibliography}{ACBFS02}

\bibitem[ACBFS02]{auer2002nonstochastic}
Peter Auer, Nicolo Cesa-Bianchi, Yoav Freund, and Robert~E Schapire.
\newblock The nonstochastic multiarmed bandit problem.
\newblock {\em SIAM journal on computing}, 32(1):48--77, 2002.

\bibitem[BBLB20]{BISTRITZ2020}
Ilai Bistritz, Tavor Baharav, Amir Leshem, and Nicholas Bambos.
\newblock My fair bandit: Distributed learning of max-min fairness with
  multi-player bandits.
\newblock In {\em International Conference on Machine Learning}, pages
  930--940. PMLR, 2020.

\bibitem[CKM{\etalchar{+}}19]{caragiannis2019unreasonable}
Ioannis Caragiannis, David Kurokawa, Herv{\'e} Moulin, Ariel~D Procaccia,
  Nisarg Shah, and Junxing Wang.
\newblock The unreasonable fairness of maximum nash welfare.
\newblock {\em ACM Transactions on Economics and Computation (TEAC)},
  7(3):1--32, 2019.

\bibitem[CKSV19]{CELIS2019}
L~Elisa Celis, Sayash Kapoor, Farnood Salehi, and Nisheeth Vishnoi.
\newblock Controlling polarization in personalization: An algorithmic
  framework.
\newblock In {\em Proceedings of the conference on fairness, accountability,
  and transparency}, pages 160--169, 2019.

\bibitem[DP09]{dubhashi2009concentration}
Devdatt~P Dubhashi and Alessandro Panconesi.
\newblock {\em Concentration of measure for the analysis of randomized
  algorithms}.
\newblock Cambridge University Press, 2009.

\bibitem[EG59]{eisenberg1959consensus}
Edmund Eisenberg and David Gale.
\newblock Consensus of subjective probabilities: The pari-mutuel method.
\newblock {\em The Annals of Mathematical Statistics}, 30(1):165--168, 1959.

\bibitem[HMS21]{HOSSAIN2021}
Safwan Hossain, Evi Micha, and Nisarg Shah.
\newblock Fair algorithms for multi-agent multi-armed bandits.
\newblock {\em Advances in Neural Information Processing Systems}, 34, 2021.

\bibitem[JKMR16]{JOSEPH2016}
Matthew Joseph, Michael Kearns, Jamie~H Morgenstern, and Aaron Roth.
\newblock Fairness in learning: Classic and contextual bandits.
\newblock {\em Advances in neural information processing systems}, 29, 2016.

\bibitem[KN79]{kaneko1979nash}
Mamoru Kaneko and Kenjiro Nakamura.
\newblock The nash social welfare function.
\newblock {\em Econometrica: Journal of the Econometric Society}, pages
  423--435, 1979.

\bibitem[LLR16]{lattimore2016causal}
Finnian Lattimore, Tor Lattimore, and Mark~D Reid.
\newblock Causal bandits: Learning good interventions via causal inference.
\newblock {\em Advances in Neural Information Processing Systems}, 29, 2016.

\bibitem[LS20]{lattimore2020bandit}
Tor Lattimore and Csaba Szepesv{\'a}ri.
\newblock {\em Bandit algorithms}.
\newblock Cambridge University Press, 2020.

\bibitem[Mou04]{moulin2004fair}
Herv{\'e} Moulin.
\newblock {\em Fair division and collective welfare}.
\newblock MIT press, 2004.

\bibitem[NJ50]{nash1950bargaining}
John~F Nash~Jr.
\newblock The bargaining problem.
\newblock {\em Econometrica: Journal of the econometric society}, pages
  155--162, 1950.

\bibitem[Pet17]{peterson2017introduction}
Martin Peterson.
\newblock {\em An introduction to decision theory}.
\newblock Cambridge University Press, 2017.

\bibitem[PGNN20]{PATIL2020}
Vishakha Patil, Ganesh Ghalme, Vineet Nair, and Y~Narahari.
\newblock Achieving fairness in the stochastic multi-armed bandit problem.
\newblock In {\em Proceedings of the AAAI Conference on Artificial
  Intelligence}, volume~34, pages 5379--5386, 2020.

\bibitem[SB18]{sutton2018reinforcement}
Richard~S Sutton and Andrew~G Barto.
\newblock {\em Reinforcement learning: An introduction}.
\newblock MIT press, 2018.

\bibitem[Sli19]{slivkins2019introduction}
Aleksandrs Slivkins.
\newblock Introduction to multi-armed bandits.
\newblock {\em Foundations and Trends{\textregistered} in Machine Learning},
  12(1-2):1--286, 2019.

\bibitem[Tho33]{thompson1933likelihood}
William~R Thompson.
\newblock On the likelihood that one unknown probability exceeds another in
  view of the evidence of two samples.
\newblock {\em Biometrika}, 25(3-4):285--294, 1933.

\bibitem[Var74]{varian1974equity}
Hal~R Varian.
\newblock Equity, envy, and efficiency.
\newblock {\em Journal of Economic Theory}, 9(1):63--91, 1974.

\bibitem[You04]{young2004strategic}
H~Peyton Young.
\newblock {\em Strategic learning and its limits}.
\newblock OUP Oxford, 2004.

\end{thebibliography}
